\documentclass{article} 
\usepackage{iclr2023_conference,times}

\usepackage{amsmath,amsfonts,bm}

\def\eqref#1{equation~\ref{#1}}

\def\1{\bm{1}}

\DeclareMathAlphabet{\mathsfit}{\encodingdefault}{\sfdefault}{m}{sl}
\SetMathAlphabet{\mathsfit}{bold}{\encodingdefault}{\sfdefault}{bx}{n}

\newcommand{\R}{\mathbb{R}}

\DeclareMathOperator*{\argmax}{arg\,max}
\DeclareMathOperator*{\argmin}{arg\,min}

\usepackage{hyperref}
\usepackage{url}

\iclrfinalcopy

\author{Duc Anh Nguyen$^{*}$\\
LMU Munich\\
\And
Ron Levie$^{*}$\\
Technion \\
Israel Institute of Technology\\
\And
Julian Lienen$^{*}$\\
Paderborn University\\
\And
Gitta Kutyniok \\
LMU Munich\\
University of Tromsø\\
\And
Eyke Hüllermeier \\
LMU Munich
}

\usepackage[utf8]{inputenc} 
\usepackage[T1]{fontenc}    
\usepackage{hyperref}       
\usepackage{url}            
\usepackage{booktabs}       
\usepackage{amsfonts}       
\usepackage{nicefrac}       
\usepackage{microtype}      
\usepackage{xcolor}         

\usepackage{amsmath}
\usepackage{amssymb}
\usepackage{mathtools}
\usepackage{amsthm}
\usepackage{enumitem}

\usepackage[font=small]{caption}
\usepackage{subcaption}
\usepackage{graphicx}
\usepackage{tikz}
\usepackage{multirow}
\usepackage{wrapfig}

\usepackage{cleveref}
\crefformat{equation}{(#2#1#3)}
\crefrangeformat{equation}{(#3#1#4) to~(#5#2#6)}
\crefmultiformat{equation}{(#2#1#3)}%
{ and~(#2#1#3)}{, (#2#1#3)}{ and~(#2#1#3)}

\newcommand{\N}{\mathbb{N}}

\newcommand{\calS}{\mathcal{S}}
\newcommand{\calP}{\mathcal{P}}
\newcommand{\calF}{\mathcal{F}}
\newcommand{\calL}{\mathcal{L}}

\newcommand{\norm}[1]{\left\Vert #1 \right\Vert}
\newcommand{\set}[1]{\left\lbrace #1\right\rbrace}

\newcommand{\sprod}[1]{\left\langle #1 \right\rangle}

\renewcommand{\vec}[1]{\boldsymbol{#1}}
\renewcommand{\hbar}{\bar{\vec{h}}}
\newcommand{\SigmaW}{\vec{\Sigma}_W}
\newcommand{\SigmaB}{\vec{\Sigma}_B}
\newcommand{\trace}{\text{trace}}
\newcommand{\mem}{\text{mem}}

\theoremstyle{plain}
\newtheorem{theorem}{Theorem}[section]

\newtheorem{lemma}[theorem]{Lemma}

\theoremstyle{definition}
\newtheorem{definition}[theorem]{Definition}
\newtheorem{assumption}[theorem]{Assumption}
\theoremstyle{remark}
\newtheorem{remark}{Remark}

\title{Memorization-Dilation: \\ Modeling Neural Collapse Under Label Noise}

\begin{document}

\maketitle
{\def \thefootnote{*} \footnotetext{These authors contributed equally to this work.}}

\begin{abstract}
    The notion of neural collapse refers to several emergent phenomena that have been empirically observed across various canonical classification problems.  During the terminal phase of training a deep neural network, the feature embedding of all examples of the same class tend to collapse to a single representation, and the features of different classes tend to separate as much as possible. Neural collapse is often studied through a simplified model, called the layer-peeled model, in which the network is assumed to have ``infinite expressivity'' and can map each data point to any arbitrary representation. In this work we study a more realistic variant of the layer-peeled model, which takes the positivity of the features into account. Furthermore, we extend this model to also incorporate the limited expressivity of the network. Empirical evidence suggests that the memorization of noisy data points leads to a degradation (dilation) of the neural collapse. Using a model of the memorization-dilation (M-D) phenomenon, we show one mechanism by which different losses lead to different performances of the trained network on noisy data. Our proofs reveal why label smoothing, a modification of cross-entropy empirically observed to produce a regularization effect, leads to improved generalization in classification tasks.
\end{abstract}

\section{Introduction}\label{Sec:intro}

The empirical success of deep neural networks has accelerated the introduction of new learning algorithms and triggered new applications, with a pace that makes it hard to keep up with profound theoretical foundations and insightful explanations.  
 As one of the few yet particularly appealing theoretical characterizations of overparameterized models trained for canonical classification tasks, \textit{Neural Collapse} (NC) provides a mathematically elegant formalization of learned feature representations \citet{DBLP:journals/corr/abs-2008-08186}.  

To explain NC, consider the following setting. Suppose we are given a \emph{balanced} dataset $\mathcal{D} = \set{(\vec{x}_n^{(k)}, y_n)}_{k \in [K], n \in [N]} \subset \mathcal{X} \times \mathcal{Y}$ in the instance space $\mathcal{X} = \mathbb{R}^d$ and label space $\mathcal{Y} = [N] := \{1, \ldots, N\}$, i.e.\  each class $n \in [N]$ has exactly $K$ samples $\vec{x}_{n}^{(1)}, \hdots, \vec{x}_{n}^{(K)}$. We consider network architectures commonly used in classification tasks that are composed of a \emph{feature engineering} part $g \colon \mathcal{X} \to \R^M$ (which maps an input signal $ \vec{x} \in \mathcal{X}$ to its feature representation $g(\vec{x}) \in \R^M$) and a \emph{linear classifier} $\vec{W}(\cdot) + \vec{b}$ given by a weight matrix $\vec{W}\in \R^{N \times M}$ as well as a bias vector $\vec{b}\in \R^N$. Let $\vec{w}_n$  denote the row vector of $\vec{W}$ associated with class $n\in[N]$. During training, both classifier components are simultaneously optimized by minimizing the cross-entropy loss.

Denoting the \emph{feature representations} $g(\vec{x}_n^{(k)})$ of the sample $\vec{x}_n^{(k)}$ by $\vec{h}_n^{(k)}$, the \emph{class means} and the \emph{global mean} of the features by 
$$
\vec{h}_n := \frac{1}{K}\sum_{i=1}^{K} \vec{h}_n^{(k)}, \qquad
\vec{h} := \frac{1}{N}\sum_{n=1}^N \vec{h}_n,
$$
NC consists of the following interconnected phenomena (where the limits take place as training progresses):  
\begin{enumerate}
    \item[\textbf{(NC1)}] \textbf{Variability collapse.} For each class $n\in [N]$, we have $\frac{1}{K}\sum_{k=1}^K \norm{\vec{h}_n^{(k)} - \vec{h}_n}^2 \to 0 \, .$
    
    \item[\textbf{(NC2)}] \textbf{Convergence to simplex equiangular tight frame (ETF) structure.} For any $m,n \in [N]$ with $m \neq n$, we have
    \begin{align*}
        \quad \norm{\vec{h}_n-\vec{h}}_2 - \norm{\vec{h}_m-\vec{h}}_2 &\to 0, \text{ and} \\
        \quad \sprod{\frac{\vec{h}_n-\vec{h}}{\norm{\vec{h}_n-\vec{h}}_2}, \frac{\vec{h}_m-\vec{h}}{\norm{\vec{h}_m-\vec{h}}_2}} &\to - \frac{1}{N-1} \, .
    \end{align*}
    
    \item[\textbf{(NC3)}] \textbf{Convergence to self-duality.} For any $n \in [N]$, it holds 
    \begin{align*}
        \frac{\vec{h}_n-\vec{h}}{\norm{\vec{h}_n-\vec{h}}_2} - \frac{\vec{w}_n}{\norm{\vec{w}_n}_2} \to 0 \, .
    \end{align*}
    
    \item[\textbf{(NC4)}] \textbf{Simplification to nearest class center behavior.} For any feature representation $\vec{u} \in \R^M$, it holds
    \begin{align*}
        \argmax_{n \in [N]} \sprod{\vec{w}_n, \vec{u}} +\vec{b}_n \to \argmin_{n \in [N]} \norm{\vec{u} -\vec{h}_n}_2 \, .
    \end{align*}
\end{enumerate}

In this paper, we consider a well known simplified model, in which the features $\vec{h}_n^{(k)}$ are not parameterized by the feature engineering network $g$ but are rather free variables. This model is often referred to as \emph{layer-peeled} model or \emph{unconstrained features} model, see e.g. \cite{DBLP:journals/corr/abs-2012-08465, Fang2021ExploringDN, DBLP:journals/corr/abs-2105-02375}. However, as opposed to those contributions, in which the features $\vec{h}_n^{(k)}$ can take any value in $\R^M$, we consider here the case $\vec{h}_n^{(k)} \geq 0$ (understood component-wise). This is motivated by the fact that features are typically the outcome of some non-negative activation function, like the Rectified Linear Unit (ReLU) or sigmoid.  
 Moreover, by incorporating the limited expressivity of the network to the layer-peeled model, we propose a new model, called \emph{memorization-dilation} (MD).  Given such model assumptions, we formally prove advantageous effects of the so-called label smoothing (LS) technique \cite{DBLP:journals/corr/SzegedyVISW15} (training with a modification of cross-entropy (CE) loss), 
 in terms of generalization performance. This is further confirmed empirically.
 

\section{Related Work}
\label{sec:relatedwork}



Studying the nature of neural network optimization is challenging. In the past, a plethora of theoretical models has been proposed to do so \cite{Sun2020OptimizationFD}. These range from analyzing simple linear \cite{pmlr-v97-kunin19a,DBLP:journals/jmiv/ZhuSEW20,DBLP:conf/icml/LaurentB18} to non-linear deep neural networks \cite{DBLP:journals/corr/SaxeMG13,DBLP:conf/iclr/YunSJ18}. As one prominent framework among others, Neural Tangent Kernels \cite{DBLP:conf/nips/JacotHG18,DBLP:journals/corr/abs-2106-10165}, where neural networks are considered as linear models on top of randomized features, have been broadly leveraged for studying deep neural networks and their learning properties. 

Many of the theoretical properties of deep neural networks in the regime of overparameterization are still unexplained. Nevertheless, certain peculiarities have emerged recently. Among those, so-called ``benign overfitting'' \cite{DBLP:journals/corr/abs-1906-11300,DBLP:journals/corr/abs-2106-03212}, where deep models are capable of perfectly fitting potentially noisy data by retaining accurate predictions, has recently attracted attention. Memorization has been identified as one significant factor contributing to this effect \cite{DBLP:conf/icml/ArpitJBKBKMFCBL17,DBLP:conf/iclr/SanyalDKT21}, which also relates to our studies. Not less interesting, the learning risk of highly-overparameterized models shows a \textit{double-descent} behavior when varying the model complexity \cite{DBLP:conf/iclr/NakkiranKBYBS20} as yet another phenomenon. Lastly, the concept of NC \cite{DBLP:journals/corr/abs-2008-08186} has recently shed light on symmetries in learned representations of overparameterized models.


After laying the foundation of a rigorous mathematical characterization of the NC phenomenon by \citet{DBLP:journals/corr/abs-2008-08186}, several follow-up works have broadened the picture. As the former proceeds from studying CE loss, the collapsing behavior has been investigated for alternative loss functions. For instance, squared losses have shown similar collapsing characteristics \cite{Poggio2020GeneralizationID, DBLP:journals/corr/abs-2101-00072}, and have paved the way for more opportunities in its mathematical analysis, e.g., by an NC-interpretable decomposition \cite{DBLP:journals/corr/abs-2106-02073}. More recently, \citet{kornblith2021better} provide an exhaustive overview over several commonly used loss functions for training deep neural networks regarding their feature collapses. 


Besides varying the loss function, different theoretical models have been proposed to analyze NC. Most prominently, \textit{unconstrained feature models} have been considered, which characterize the penultimate layer activations as free optimization variables \cite{DBLP:journals/corr/abs-2011-11619,DBLP:journals/corr/abs-2012-08465,e2021emergence}. This stems from the assumption that highly overparameterized models can approximate any patterns in the feature space. While unconstrained features models typically only look at the last feature encoder layer, \textit{layer-peeling} allows for ``white-boxing'' further layers before the last one for a more comprehensive theoretical analysis \citet{Fang2021ExploringDN}. Indeed, this approach has been applied in \citet{Extended_unconstrained_features_model}, which namely extends the unconstrained features model by one layer as well as the ReLU nonlinearity. On the other hand, \citet{DBLP:journals/corr/abs-2105-02375}, \citet{Unconstrained_Layer_Peeled_Perspective} and \citet{https://doi.org/10.48550/arxiv.2203.01238} extend the unconstrained features model analysis by studying the landscape of the loss function therein and the related training dynamics. Beyond unconstrained features models, \citet{DBLP:conf/icml/ErgenP21a} introduce a convex analytical framework to characterize the encoder layers for a more profound understanding of the NC phenomenon. Referring to the implications of NC on our understanding of neural networks, \citet{Limitation_NC_transfer_learning} and \citet{DBLP:journals/corr/abs-2112-15121} discuss the impact of NC on test data in the sense of generalization and transfer learning. Finally, \citet{NC_review_paper} provides a multifaceted survey of recent works related to NC. 

\section{Layer-peeled model with positive features}\label{Section:standard_model}

As a prerequisite to the MD model, in this section we introduce a slightly modified version of the layer-peeled (or unconstrained features) model (see e.g. \cite{DBLP:journals/corr/abs-2105-02375, Fang2021ExploringDN}), in which the features have to be positive.  
 Accordingly, we will show that the global minimizers of the modified layer-peeled model correspond to an NC configuration, which differs from the global minimizers specified in other works and captures more closely the NC phenomenon in practice. 

For conciseness, we denote by $\vec{H}$ the matrix formed by the features $\vec{h}_n^{(k)}$, $n \in [N]$, $k\in[K]$ as columns, and define $\norm{\vec{W}}$ and $\norm{\vec{H}}$ to be the Frobenius norm of the respective matrices, i.e.\  $\norm{\vec{W}}^2 = \sum_{n=1}^N\norm{\vec{w}_n}^2$ and $\norm{\vec{H}}^2 = \sum_{k=1}^K \sum_{n=1}^N \norm{\vec{h}_n^{(k)}}^2$. We consider the regularized version of the model (instead of the norm constraint one as in e.g. \cite{Fang2021ExploringDN}) \footnote[1]{Note that for simplicity we assume that the last layer does not have bias terms, i.e. $b=0$. The result can be however easily extended to the more general case when the biases do not vanish. Namely, in presence of bias terms, the statement of Theorem \ref{theorem:standard_model} and also its proof remain unchanged. }

\begin{align}\label{Problem:standard}\tag{$\calP_\alpha$}
    \begin{split}
        &\min_{\vec{W},\vec{H}} \quad  \calL_\alpha(\vec{W},\vec{H}) := L_\alpha(\vec{W},\vec{H}) + \lambda_W \norm{\vec{W}}^2 + \frac{\lambda_H}{K} \norm{\vec{H}}^2 \\
    &\text{s.t. } \quad \vec{H} \geq 0,
    \end{split}
\end{align}
where $\lambda_W, \lambda_H >0$ are the penalty parameters for the weight decays. By $L_\alpha$ we denote empirical risk with respect to the LS loss with parameter $\alpha \in [0,1)$, where $\alpha = 0$ corresponds to the conventional CE loss. More precisely, given a value of $\alpha$, the LS technique then defines the label assigned to class $n \in [N]$ as the following probability vector:
\begin{align*}
    \vec{y}_n^{(\alpha)} = (1-\alpha) \vec{e}_n + \frac{\alpha}{n} \vec{1}_N \in [0,1]^N,
\end{align*}
where $\vec{e}_n \in \R^N$ denotes the $n$-th standard basis vector and $\vec{1}_N\in \R^N$ denotes the vector consisting of only ones. Let $p: \R^M \to \R^N$ be the function that assigns  to each feature representation $\vec{z}\in \R^M$ the probability scores of the classes (as a probability vector in $\R^N$),
\begin{align*}
    p_{\vec{W}}(\vec{z}) := \operatorname{softmax}(\vec{W}\vec{z}) := \Bigg[ \frac{e^{\sprod{\vec{w}_m,\vec{z}}}}{\sum_{i=1}^N e^{\sprod{\vec{w}_i,\vec{z}}}}\Bigg]_{m=1}^N \in [0,1]^N.
\end{align*}
Then the LS loss corresponding to a sample in class $n \in [N]$ is given by
\begin{align}
\label{LS_loss}
    \ell_{\alpha}(\vec{W},\vec{z},\vec{y}_n^{(\alpha)}) := \sprod{-\vec{y}_{n}^{(\alpha)}, \log p_{\vec{W}}(\vec{z})} := \sum_{m=1}^N -\vec{y}_{nm}^{(\alpha)} \log \big(p_{\vec{W}}(\vec{z})_m\big)
\end{align}
and the LS empirical risk $L_\alpha$ is defined as
\begin{align*}
    L_\alpha(\vec{W},\vec{H}) &= \frac{1}{NK} \sum_{k=1}^K \sum_{n=1}^N \ell_\alpha\Big(\vec{W}, \vec{h}_n^{(k)}, \vec{y}_n^{(\alpha)} \Big).
\end{align*}

We will show that in common settings, the minimizers of (\ref{Problem:standard} ) correspond to \emph{neural collapse (NC) configurations}, which we formalize in Def. \ref{def_NC} below. 

\begin{definition}[NC configurations]\label{def_NC}
Let $K,M,N\in \N$, $M\geq N$. A pair $(\vec{W},\vec{H})$ of a weight matrix formed by rows $\vec{w}_n \in \R^M$ and a feature matrix formed by columns $\vec{h}_n^{(k)} \in \R_+^M$ (with $n\in [N], k\in [K]$) is said to be a \emph{NC configuration} if 
\begin{enumerate}[label=(\roman*)]
    \item The feature representations $\vec{h}_n^{(k)}$ within every class $n \in [N]$ are equal for all $k\in [K]$, and thus equal to their class mean $\vec{h}_n := \frac{1}{K}\sum_{k=1}^K \vec{h}_n^{(k)}$.
    \item The class means $\{\vec{h}_n\}_{n=1}^N$ have equal norms and form an (entry-wise) non-negative orthogonal system.
    \item
    Let $P_{\vec{h}^{^\perp}}$ be the projection upon the subspace of $\R^M$ orthogonal to $\vec{h}=\frac{1}{N}\sum_{n=1}^N h_n$. Then for every $n\in [N]$, it holds $\vec{w}_n= CP_{\vec{h}^{^\perp}}\vec{h}_n\,$
    for some constant $C$ independent of $n$. 
\end{enumerate}

\end{definition}

Our main theorem in this section can be represented as follows.
\begin{theorem}\label{theorem:standard_model}
Let $M \geq N$, $\alpha \in [0,1)$. Assume that $\frac{N-1}{N} \alpha +2\sqrt{(N-1)\lambda_W\lambda_H}<1$. Then  any global minimizer of the problem (\ref{Problem:standard}) is a NC configuration.
\end{theorem}

Note that the NC configurations defined in \Cref{def_NC} above differ significantly from the ones specified in other works, e.g. \cite{Fang2021ExploringDN,DBLP:journals/corr/abs-2105-02375, all_losses_equal_2022} or \cite{Extended_unconstrained_features_model}, see Appendix B.1 for more discussion.

\section{The Memorization-Dilation model}

\subsection{Experimental Motivation}
\label{sec:quantification_mem_test_collapse}

Previous studies of the NC phenomenon mainly focus on the collapsing variability of \textit{training} activations, and make rather cautious statements about its effects on generalization. For instance, \citet{DBLP:journals/corr/abs-2008-08186} report slightly improved test accuracies for training beyond zero training error. Going a step further, \citet{DBLP:journals/corr/abs-2105-02375} show that the NC phenomenon also happens for overparameterized models when labels are completely randomized. Here, the models seem to \textit{memorize} by overfitting the data points, however, a rigorous study how label corruption affects generalization in the regime of NC is still lacking.

To fill the gap, we advocate to analyze the effects of label corruption in the training data on the (previously unseen) \textit{test} instead of the training feature collapse. 
Eventually, tight test class clusters go hand in hand with easier separation of the instances and, thus, a smaller generalization error. Following \citet{DBLP:journals/corr/abs-2105-02375}, we measure the collapse of the penultimate layer activations by the $\mathcal{NC}_1$ metric. This metric depicts the relative magnitude of the within-class covariance $\SigmaW$ with respect to the between-class covariance $\SigmaB$ of the penultimate layer features and is defined as
\begin{equation}
    \mathcal{NC}_1 := \frac{1}{N} \trace(\SigmaW \SigmaB^\dagger),
    \label{metric:nc1}
\end{equation}
where 
\begin{equation*}
    \SigmaW := \frac{1}{NK} \sum_{n=1}^N \sum_{k=1}^{K} (\vec{h}^{(k)}_n - \vec{h}_n) (\vec{h}^{(k)}_n - \vec{h}_n)^\top \in \R^{M\times M}\, ,
\end{equation*}
\begin{equation*}
    \SigmaB := \frac{1}{N} \sum_{n=1}^N (\vec{h}_n - \vec{h}) (\vec{h}_n - \vec{h})^\top \in \R^{M \times M},
\end{equation*}
and $\SigmaB^\dagger$ denotes the pseudo-inverse of $\SigmaB$. Here, we adopt the notations from Section \ref{Sec:intro}: $\vec{h}_n^{(k)}\in \R^M$ denotes the feature representation of $k$-th sample in class $n$, $\vec{h}_n$ the class mean and $\vec{h}$ the global mean. Moreover, we distinguish $\mathcal{NC}_1^{\text{train}}$ and $\mathcal{NC}_1^{\text{test}}$ to be calculated on the training and test instances, respectively. We call $\mathcal{NC}_1^{\text{test}}$  \textit{dilation}.


Let us now turn to the notion of \textit{memorization}, which is not uniquely defined in deep learning literature. 
Here, we define memorization in the context of the NC setting and in a global manner, different from other works, e.g. \citet{DBLP:conf/nips/FeldmanZ20}. Formally, suppose that label noise is incorporated by (independently) corrupting the instance of each class label $n$ in the training data with probability $\eta \in (0,1)$, where corruption means drawing a label uniformly at random from the label space $\mathcal{Y}$. We denote the set of corrupted instances by $[\widetilde{K}]$.
For a given dataset $\mathcal{D}$ (with label noise $\eta$), we define \emph{memorization} as 
\begin{equation}
    \label{eq:mem}
    \mem := \sum_{n=1}^N \sum_{k\in [\widetilde{K}]} \| \vec{h}^{(k)}_{n} - \vec{h}^*_n \|_2 \, ,
\end{equation}
where $\vec{h}^*_n$ denotes the mean of (unseen) test instances belonging to class $n$.

We call the original ground truth label of a sample its \emph{true label}. We call the label after corruption, which may be the true label or not, the \emph{observed label}.
Since instances of the same true label tend to have similar input features in some sense, the network is biased to map them to similar feature representations. Instances are corrupted randomly, and hence, instances of the same true label but different observed labels do not have predictable characteristics that allow the network to separate them in a way that can be generalized. When the network nevertheless succeeds in separating such instances, we say that the network \emph{memorized} the feature representations of the corrupted instances in the training set. The metric $\text{mem}$ in (\ref{eq:mem}) thus measures memorization. 
The above memorization also affects dilation. Indeed, the network uses the feature engineering part to embed samples of similar features (that originally came from the same class), to far apart features, that encode different labels. Such process degrades the ability of the network to embed samples consistently, and leads to dilation. 

To quantify the interaction between $\text{mem}$ and $\mathcal{NC}_1^{\text{test}}$, we analyzed the learned representations $\vec{h}$ in the penultimate layer feature space for different noise configurations. One may wonder whether one can see a systematic trend in the test collapse given the memorization, and how this evolves over different loss functions.

To this end, we trained simple multi-layer neural networks for two classes ($N=2$), which we subsampled from the image classification datasets MNIST \citet{MNIST}, FashionMNIST \citet{FASHION_MNIST}, CIFAR-10 \citet{CIFAR} and SVHN \citet{svhn}. The labels are corrupted with noise degrees $\eta \in [0.025, 0.4]$. The network consists of $9$ hidden layers with $2048$ neurons each, thus, it represents a vastly overparameterized model. The feature dimension $M$ is set to the number of classes $N$. We trained these networks using the CE and LS loss with a smoothing factor $\alpha = 0.1$, as well as the mean-squared error (MSE). Moreover, we consider label relaxation (LR) \citet{DBLP:conf/aaai/LienenH21} as a generalization to LS with a relaxation degree $\alpha = 0.1$. The networks were trained until convergence in $200$ epochs (where the last $50$ epochs did not make any significant changes) using SGD with an initial learning rate of $0.1$ multiplied by $0.1$ each $40$ epochs and a small weight decay of $0.001$. Moreover, we considered ReLU as activation function throughout the network, as well as batch normalization in each hidden layer. A linear softmax classifier is composed on the encoder. We conducted each experiment ten times with different seeds.



\begin{figure}[t]
\centering
\begin{subfigure}{\columnwidth}
  \centering
  \includegraphics[width=\linewidth]{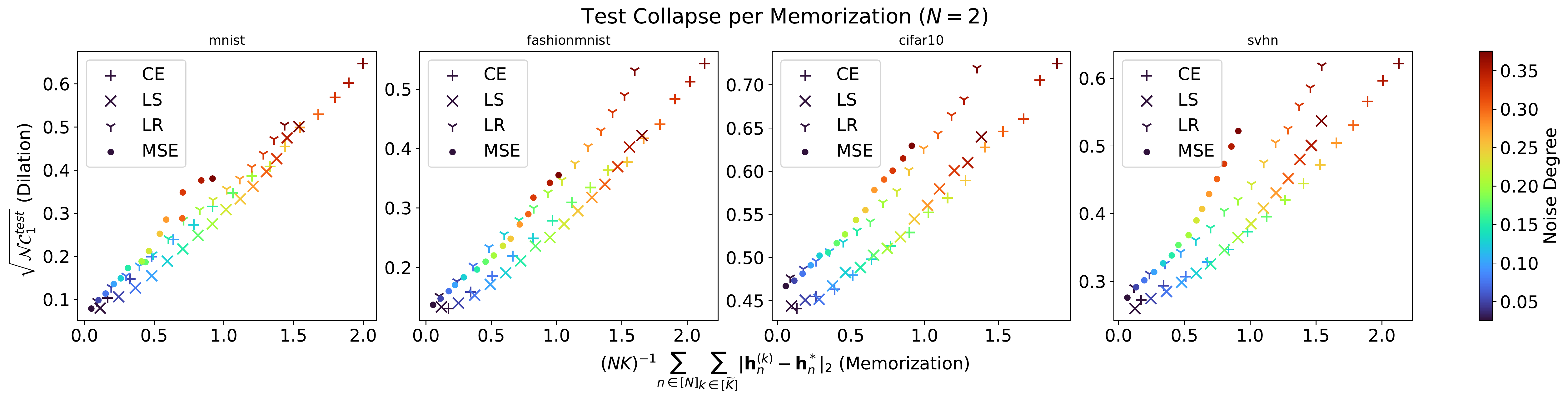}
\end{subfigure}%
\\
\begin{subfigure}{\columnwidth}
  \centering
  \includegraphics[width=\linewidth]{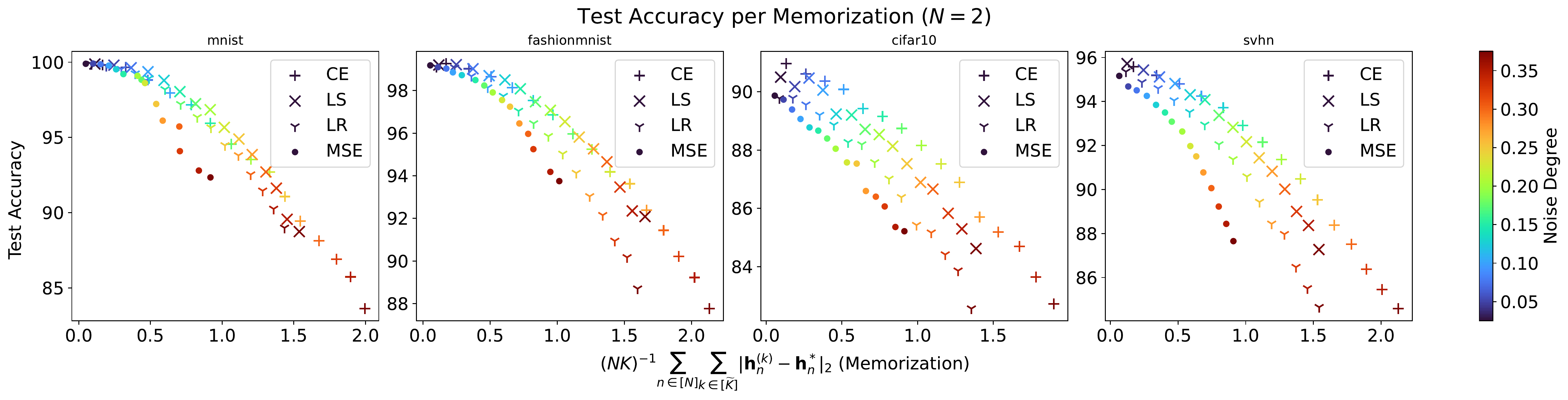}
\end{subfigure}
    \caption{Feature collapse of the test instances in terms of $\sqrt{\mathcal{NC}^{\text{test}}_1}$ per memorization (top row) and the resulting test accuracies (bottom row) averaged over ten random seeds. Comparing the markers of the same color, it can be observed that LS consistently performs better than CE across all datasets, with very few exceptions (the very low noise degrees in cifar10).}
    \label{fig:mem_exps_c2}
\end{figure}

The results for the above experimental setting are shown in Fig.\ \ref{fig:mem_exps_c2}, in which one can observe the trends of $\sqrt{\mathcal{NC}^{\text{test}}_1}$ per memorization for various configurations. As can be seen, the figure shows an approximately linear correspondence between $\sqrt{\mathcal{NC}^{\text{test}}_1}$ and $\text{mem}$ for the CE derivatives (CE and LS) on all datasets when $\text{mem}$ is not large. 

\begin{figure}[h]
\centering
\begin{subfigure}{0.45\columnwidth}
  \centering
  \includegraphics[width=\linewidth]{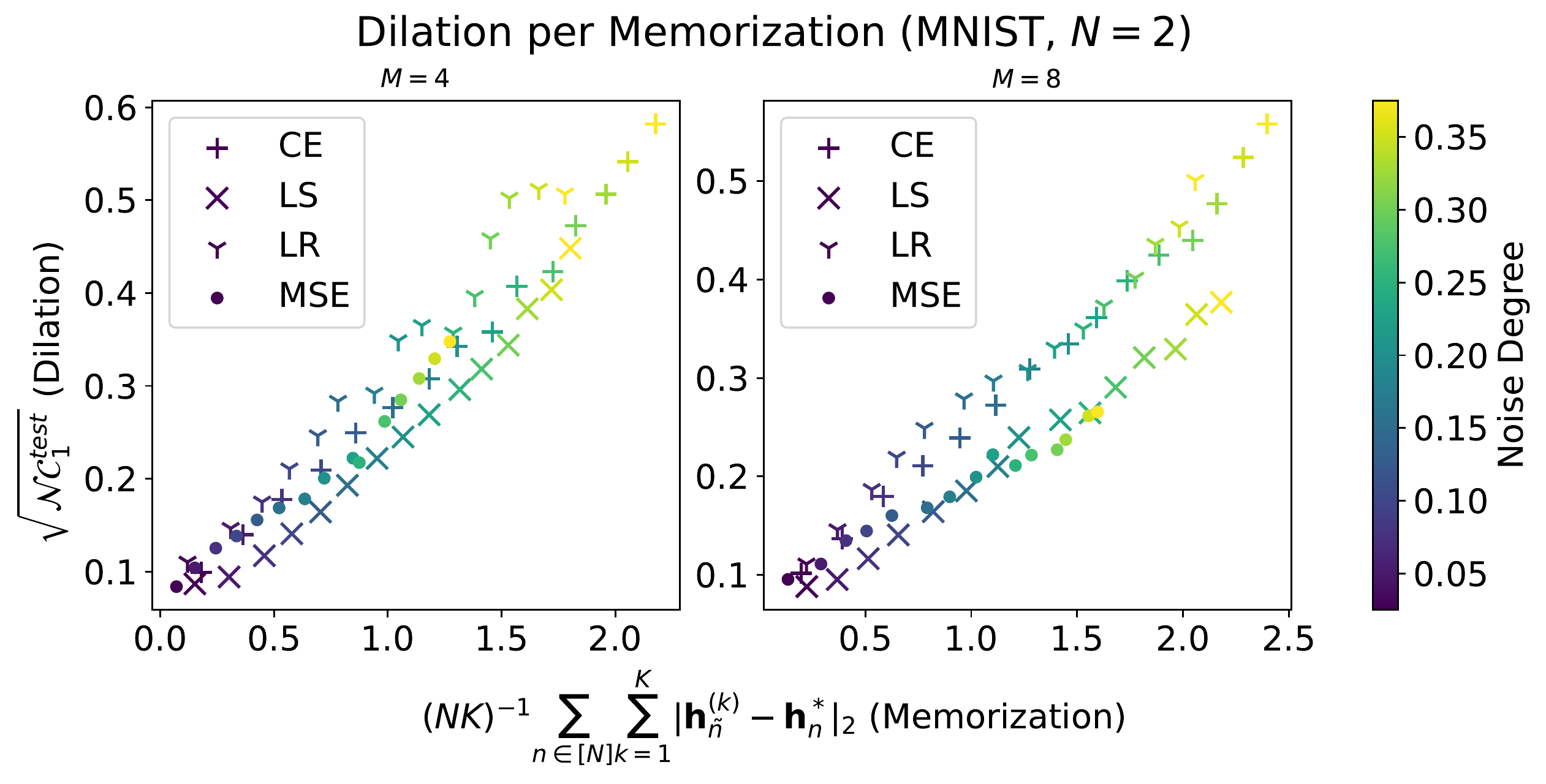}
\end{subfigure}%
\hspace{1cm}
\begin{subfigure}{0.45\columnwidth}
  \centering
  \includegraphics[width=\linewidth]{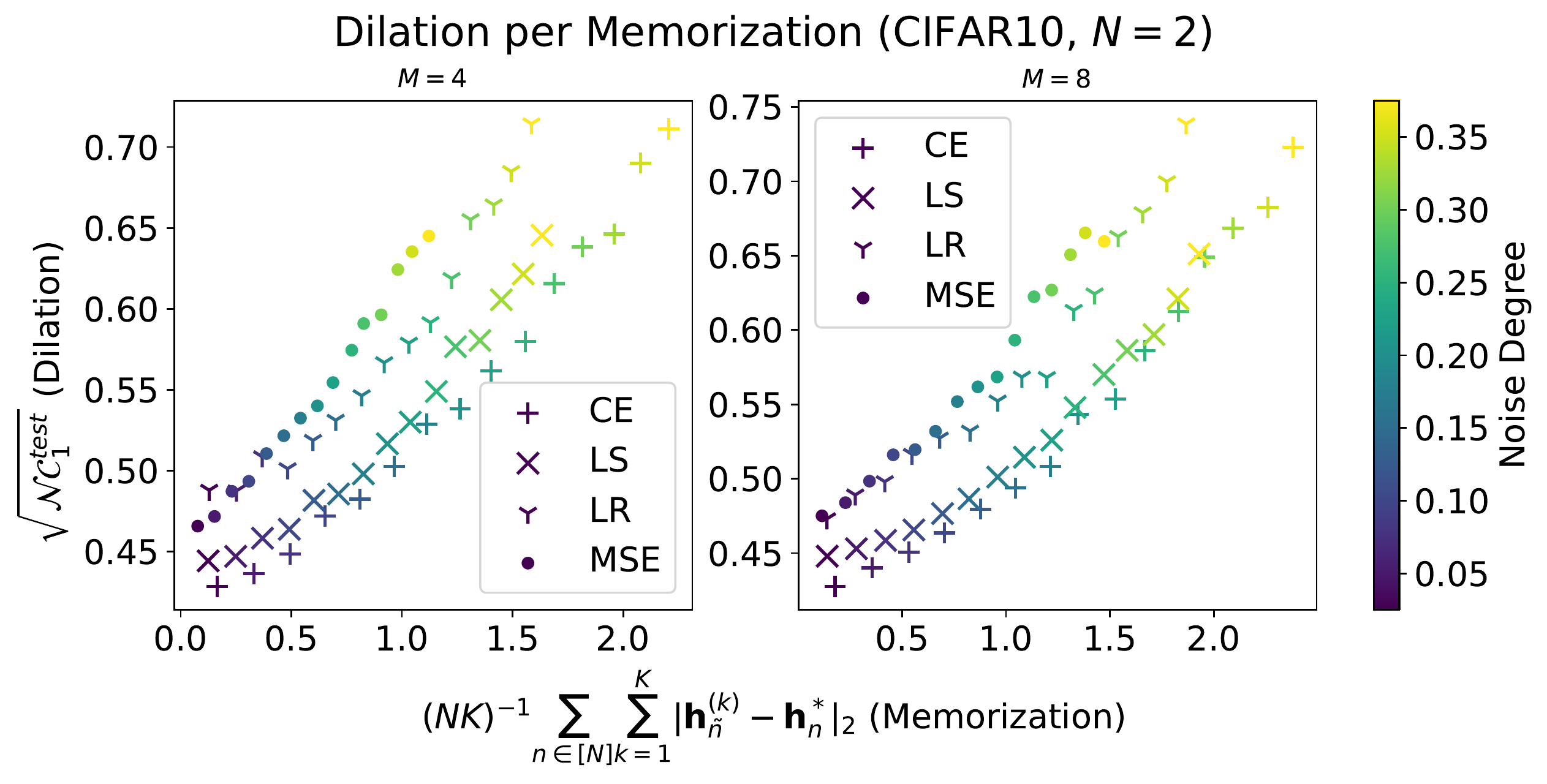}
\end{subfigure}
    \caption{Feature collapse of the test instances in terms of $\sqrt{\mathcal{NC}^{\text{test}}_1}$ per memorization as the feature dimension $M$ varies.}
    \label{fig:new_mem_exps_c2}
\end{figure}
Moreover, as CE and LS share the same slope, these results suggest that the degradation of the test collapse (aka dilation) is a function of memorization and the network expressitivity, and not of the choice of the loss. The loss only affects how the noise translates to memorization, but not how memorization translates to dilation. Even though the same amount of noise is mapped to different memorization values in CE and LS, the memorization-dilation curve is nevertheless shared between CE and LS. 
Hence, since LS leads the network to memorize less, it results in improved performance (cf. Fig. \ref{fig:mem_exps_c2}). We can further see that MSE and LR show a different memorization-dilation correspondence, which means that these losses affect the inductive bias in a different way than CE and LS. 

We repeated the experiments for different values of the feature dimension $M$ and show the example results in Fig. \ref{fig:new_mem_exps_c2}. Here, one can see the similar trends of dilation per memorization as before. In the appendix, we provide additional results showing the behavior in the multi-class case $N>2$ with different models for label noise. The results support our MD model, and show that the memorization-dilation curve is roughly independent of the noise model for low-to-mid noise levels. 

\subsection{The Memorization-Dilation Model}

Motivated by the observations of the previous experiments, we propose the so-called \textit{memorization-dilation (MD) model}, which extends the unconstrained feature model by incorporating the interaction between memorization and dilation as a model assumption. By this, we explicitly capture the limited expressivity of the network, thereby modeling the inductive bias of the underlying model.

This model shall provide a basis to mathematically characterize the difference in the learning behavior of CE and 
 LS. More specifically, we would like to know \emph{why}  LS shows improved generalization performance over conventional CE, as was observed in past works \citet{DBLP:conf/nips/MullerKH19}. The main idea can be explained as follows. We first note that dilation is directly linked to generalization (see also \citet{kornblith2021better}), since the more concentrated the feature representations of each class are, the easier it is to separate the different classes with a linear classifier without having outliers crossing the decision boundary. 
 The MD model asserts that dilation is a linear function of memorization. Hence, the only way that LS can lead to less dilation than CE, is if LS memorizes less than CE.
 Hence, the goal in our analysis is to show that, under the MD model,  LS indeed leads to less memorization than CE. 
 Note that this description is observed empirically in the experiments of Section \ref{sec:quantification_mem_test_collapse}. 
 
 Next we define the MD model in the binary classification setting.
\begin{definition}\label{def:MD_model}
We call the following minimization problem $\mathcal{MD}$.
Minimize the \emph{MD risk} 
\begin{align*}
    \mathcal{R}_{\lambda,\eta,\alpha}(\vec{U},r):=
    F_{\lambda, \alpha}(\vec{W},\vec{H},r) + \eta G_{\lambda,\alpha}(\vec{W},\vec{U},r),
\end{align*}
with respect to the \emph{noisy feature embedding} $\vec{U} = [\vec{u_1}, \vec{u_2}]\in \R_+^{2\times M}$  and the \emph{standard deviation} $r\geq 0$, under the constraints
\begin{align}
    \eta\norm{\vec{h}_1-\vec{u}_2} \leq \frac{C_{MD} r}{ \norm{\vec{h}_1-\vec{h}_2}} \label{eq:MD_condition1}\\
    \eta\norm{\vec{h}_2-\vec{u}_1} \leq \frac{C_{MD} r}{ \norm{\vec{h}_1-\vec{h}_2}} \label{eq:MD_condition2}.
\end{align}

Here,
\begin{itemize}
    \item $\vec{H} \in \R_+^{2\times M}$ and $\vec{W}\in \R^{M\times2}$ form an NC configuration (see \Cref{def_NC}). 

    \item $C_{MD}>0$ is called the \emph{memorization-dilation slope},  
    $0\leq\alpha<1$ is called the \emph{LS parameter}, $\eta>0$ the \emph{noise level}, and   $\lambda>0$ the \emph{regularization parameter}.
    
    \item $F_{\lambda, \alpha}$ is the component in the (regularized) risk that is associated with the correctly labeled samples, 
    \begin{align*}
        \begin{split}
            F_{\lambda,\alpha} (\vec{W}, \vec{H},r) :=  &\int \Bigg(\ell_\alpha \Big(\vec{W},\vec{h}_1 + \vec{v}, \vec{y}_1^{(\alpha)} \Big)+ \lambda \norm{\vec{h}_1+\vec{v}}^2 \Bigg) d\mu_r^1(\vec{v}) \\
            &+ \int  \Bigg(\ell_\alpha \Big(\vec{W},\vec{h}_2 + \vec{v}, \vec{y}_2^{(\alpha)} \Big) + \lambda \norm{\vec{h}_2+\vec{v}}^2 \Bigg)d\mu_r^2(\vec{v})
        \end{split}
    \end{align*}
    where $\mu_r^1$ and $\mu_r^2$ are some probability distributions with mean $0$ and standard deviation $r$, and $l_{\alpha}$ is the LS loss defined in (\ref{LS_loss}). 

    \item $G_{\lambda, \alpha}$ is the component in the (regularized) risk that is associated with the corrupted samples, defined as 
    \begin{align*}
        G_{\lambda,\alpha}(\vec{W},\vec{U},r) = \ell_\alpha\Big(\vec{W}, \vec{u_1}, \vec{y}_1^{(\alpha)} \Big) + \ell_\alpha\Big(\vec{W}, \vec{u_2}, \vec{y}_2^{(\alpha)} \Big) + \lambda \norm{\vec{u_1}}^2 + \lambda \norm{\vec{u_2}}^2.
    \end{align*}
    
    
\end{itemize}

\end{definition}

The MD model can be interpreted as follows. First we consider the feature representations of the correctly labeled samples in each class as samples from a distribution (namely $\mu^{1,2}_r$ in Def. \ref{def:MD_model}) with standard deviation $r$, a parameter that measures the dilation of the class cluster. In a natural way, the corresponding risk $F_{\lambda,\alpha}$ involves the loss average over all samples, i.e. the loss integral over the distribution. For simplicity, we assume that the class centers $\vec{h}_1, \vec{h}_2$ as well as the weight matrix $\vec{W}$ are fixed as described by the NC configuration. This is a reasonable simplification as it has been always observed in the experiments.

On the other hand, the feature representations of corrupted samples are $\vec{u}_1$ and $\vec{u}_2$.\footnote[2]{Certainly one can, instead of two single points $\vec{u}_1$ and $\vec{u}_2$, two distributions centered around $\vec{u}_1$ and $\vec{u}_2$, similarly as before for uncorrupted samples. However, it is quite straightforward to see that the minimization of the MD risk over the dilation of these two distributions is independent of other variables (not like $r$), and thus the minimum should be attained in the case of collapsing into two single points. Thus, for convenience we assume directly here that $G_{\lambda,\alpha}$ involves only two single points.} The amount of memorization in the first class is defined to be $\eta\|\vec{h}_2-\vec{u}_1\|$, since the more noise $\eta$ there is, the more examples we need to memorize. The amount of memorization in the second class is defined the same way. The (normalized) dilation is defined to be $\frac{ r}{ \norm{\vec{h}_1-\vec{h}_2}}$, which models a similar quantity to  (\ref{metric:nc1}). 
\begin{wrapfigure}{r}{4.4cm}
    \centering
    \includegraphics[width=\linewidth]{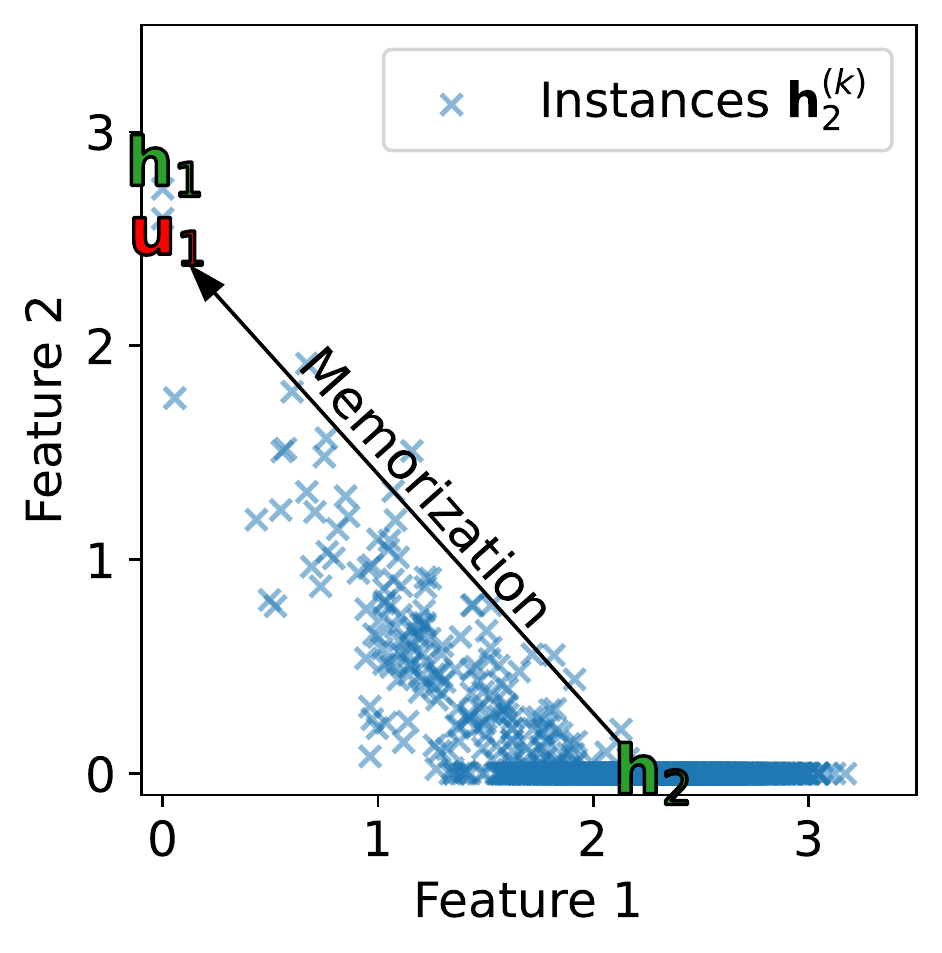}
    \caption{Exemplary illustration of the MD model for a MLP network trained on MNIST. The instances $\vec{h}_2^{(k)}$ are test images correctly labeled as $1$, with centroid $\vec{h}_2$. The centroid of the test images with correct label $0$ is $\vec{h}_1$. The centroid of training images which were originally labeled as $1$ but are mislabeled as $0$ is $\vec{u}_1$. The memorization of $\vec{u}_1$ moves it close to $\vec{h}_1$, and causes dilation of the instances $\vec{h}_2^{(k)}$.}
    \label{fig:mem_dil}
\end{wrapfigure}
The constraints (\ref{eq:MD_condition1}) and (\ref{eq:MD_condition2}) tell us that in order to map noisy samples $\vec{u}_1$ away from $\vec{h}_2$, we have to pay with dilation $r$. The larger $r$ is, the further away we can map $\vec{u}_1$ from $\vec{h}_2$.  The correspondence between memorization  and dilation  is linear with slope $C_{MD}$ by assumption. There are two main forces in the optimization problem: $\vec{u}_1$ would like to be as close as possible to its optimal position $\vec{h}_1$, and similarly $\vec{u}_2$ likes to be close to $\vec{h}_2$. In view of the constraints (\ref{eq:MD_condition1}) and (\ref{eq:MD_condition2}), to achieve this, $r$ has to be increased to $r_{\max} := \frac{\eta \norm{\vec{h}_1-\vec{h}_2}^2}{C_{MD}}$. On the other hand, the optimal $r$ for the term $F_{\lambda,\alpha}$ is $r=0$, namely, the layer-peeled NC configuration. An optimal solution hence balances between memorization and dilation. See Fig. \ref{fig:mem_dil} for a visualization of the MD model.

Our goal in this section is to compare the optimal value $r$ in case of LS and CE losses. We will distinguish between these two cases by setting the value of $\alpha$ in the MD model to $0$ for CE and to some $\alpha_0>0$ for LS. This will result in two different scales of the feature embeddings $\vec{H}$, denoted by $\vec{H}^{CE}$ and $\vec{H}^{LS}$ for CE and LS loss respectively, with the ratio
\begin{equation}
\label{eq:gamma}
    \gamma := \norm{\vec{H}^{CE}} / \norm{\vec{H}^{LS}} >1 ,
\end{equation}
which holds under the reasonable assumption that the LS technique is sufficiently effective, or more precisely $\alpha_0 > 2\sqrt{\lambda_W \lambda_H}$. 

The main result in this section will be Theorem \ref{theorem:MD_model}, which states informally that in the low noise regime, the optimal dilation in case of LS loss is smaller than that in case of CE loss. Before presenting this theorem, we will first establish several assumptions on the distributions $\mu_r^{1,2}$ and the noise $\eta$ in Assumption \ref{Assumption:MD_model}. Basically we allow a rich class of distributions and only require certain symmetry and bounded supports in terms of $r$, as well as require $\eta$ to be small in terms of the ratio $\gamma$. 


\begin{assumption}\label{Assumption:MD_model}
\text{ }
\begin{enumerate}
    \item Let $\alpha_0 >0$. We assume that the solution of
    \begin{align*}
        &\min_{\vec{W},\vec{H}} \quad \ell_\alpha \Big(\vec{W},\vec{h}_1, \vec{y}_1^{(\alpha)} \Big)+ \ell_\alpha \Big(\vec{W},\vec{h_2}, \vec{y}_2^{(\alpha)}\Big) + \lambda_W \norm{\vec{W}}^2 + \lambda_H \norm{\vec{H}}^2\\
        &\text{ s.t. } \quad \vec{H} \geq 0 \, .
    \end{align*}
    is given by $(\vec{W},\vec{H}) = (\vec{W}^{CE}, \vec{H}^{CE})$ for $\alpha = 0$ and $(\vec{W},\vec{H}) = (\vec{W}^{LS}, \vec{H}^{LS})$ for $\alpha = \alpha_0$.

    \item Assume that the distributions $\mu_r^1$ and $\mu_r^2$ are \emph{centered}, in the sense that
    \begin{align*}
        \int \sprod{\vec{w}_2-\vec{w}_1,\vec{v}} d\mu_r^1(\vec{v}) = \int \sprod{\vec{w}_1-\vec{w}_2,\vec{v}} d\mu_r^2(\vec{v}) = 0 \, ,\\
        \int \sprod{\vec{h}_1,\vec{v}} d\mu_r^1(\vec{v}) = \int \sprod{\vec{h}_2,\vec{v}} d\mu_r^2(\vec{v}) =0 \, .
    \end{align*}
    Furthermore, we assume that there exists a constant $A>0$ such that $\norm{\vec{v}}\leq Ar$ for any vector $\vec{v}$ that lies in the support of $\mu_r^1$ or in the support of $\mu_r^2$. 
    
    
    \item Assume that the noise level $\eta$ and the LS parameter $\alpha_0$ satisfy the following. We suppose $\alpha_0>4\sqrt{\lambda_W \lambda_H}$, which guarantees $\gamma := \norm{\vec{H}^{CE}} / \norm{\vec{H}^{LS}} >1$. We moreover suppose that   $\eta$ is sufficiently small to guarantee $\eta^{1/2} <\tilde{C} (1- \frac{1}{\gamma})$ where $\tilde{C}:= \frac{C_{MD}}{\sqrt{2}\norm{\vec{h}_1^{CE}-\vec{h}^{CE}_2}}$. 
\end{enumerate}
\end{assumption}

Now our main result in this section can be formally stated as below.
\begin{theorem}\label{theorem:MD_model}
Suppose that \Cref{Assumption:MD_model} holds true for $M\geq N=2$ and $\lambda := \lambda_H$. Let $r^{CE}_*$ and $r^{LS}_*$ be the optimal dilations, i.e.\ the optimum $r$ in the $\mathcal{MD}$ problem, corresponding to the CE and LS loss (accordingly $\alpha =0$ and $\alpha = \alpha_0$), respectively. Then it holds that
\begin{align*}
    \frac{ r_*^{CE}}{ \norm{\vec{h}_1^{CE}-\vec{h}_2^{CE}}} > \frac{ r_*^{LS}}{ \norm{\vec{h}^{LS}_1-\vec{h}^{LS}_2}} \, .
\end{align*}
\end{theorem}

Theorem \ref{theorem:MD_model} reveals a mechanism by which LS achieves better generalization than CE. It is proven that LS memorizes and dilates less than CE,  which is associated with better generalization. Note that in practice, the data often have noise in the sense that not all examples are perfectly labeled. More importantly, examples from different classes may share many similarities, a situation that is also covered by the MD model: the feature representations of samples from those classes are biased toward each other. In this case, LS also leads to decreased dilation which corresponds to better class separation and higher performance \citet{kornblith2021better}.

 Interestingly, the concurrent work \citet{all_losses_equal_2022} has shown that in the noiseless setting CE and LS lead to largely identical test accuracy, which seems to contradict the statement that LS performs better claimed by our work as well as many others, e.g. \citet{kornblith2021better,DBLP:conf/nips/MullerKH19}. However, note that \citet{all_losses_equal_2022} requires the network to be sufficiently large so that it has enough expressive power to fit the underlying mapping from input data to targets, as well as to be trained until convergence. While the latter is easy to obtain, it is difficult even to check if the first requirement holds. The difference between the two results is hence possibly caused by the effect of noise and by the network expressivity: while we aim to model the limited expressivity by the MD relation, \citet{all_losses_equal_2022} focuses on networks with approximately infinite expressivity.





The MD model combines a statistical term $F_{\lambda, \alpha}$, that describes the risk over the distribution of feature embeddings of samples with clean labels, and an empirical term $\eta G_{\lambda,\alpha}$ that describes the risk over training samples with noisy labels. One point of view that can motivate such a hybrid statistical-empirical definition is the assumption that the network only memorizes samples of noisy labels, but not samples of clean labels. Such a memorization degrades (dilates) both the collapse of the training and test samples, possibly with different memorization-dilation slopes.
However, memorization is not limited to corrupted labels, but can also apply to samples of clean labels \citet{DBLP:conf/nips/FeldmanZ20}, by which the learner can partially negate the dilation of the training features (but not test features). The fact that our model does not take the memorization of clean samples into account is one of its limitations. We believe that future work should focus on modeling the total memorization of all examples. Nevertheless, we believe that our current MD model has merit, since 1) noisy labels are memorized more than clean labels, and especially in the low noise regime the assumption of observing memorization merely for corrupted labels appears reasonable, and 2) our approach and proof techniques can be the basis of more elaborate future MD models. 

\section{Conclusion}
In this paper, we first characterized the global minimizers of the Layer-Peeled Model (or the Unconstrained Features Model) with the positivity condition on the feature representations. Our characterization shows some distinctions from the results that haven been obtained in recent works for the same model without feature positivity. Besides the conventional cross-entropy (CE) loss, we studied the model in case of the label smoothing (LS) loss, showing that NC also occurs when applying this technique. 

Then we extended the model to the so-called Memorization-Dilation (MD) Model by incorporating the limited expressivity of the network. Using the MD model, which is supported by our experimental observations, we show that when trained with the LS loss, the network memorizes less than when trained by the CE loss. This poses one explanation to the improved generalization performance of the LS technique over the conventional CE loss. 

Our model has limitations, however, namely that it is limited to the case of two classes. Motivated by promising results on the applicability of our model to the multi-class setting, we believe that future work should focus on extending the MD model in this respect. With such extensions, memorization-dilation analysis has the potential to underlie a systematic comparison of the generalization capabilities of different losses, such as CE, LS, and label relaxation, by analytically deriving formulas for the amount of memorization associated with each loss. 

\subsubsection*{Acknowledgments}

This work was partially supported by the German Research Foundation (DFG) within the Collaborative Research Center ``On-The-Fly Computing'' (CRC 901 project no.~160364472). Moreover, the authors gratefully acknowledge the funding of this project by computing time provided by the Paderborn Center for Parallel Computing (PC$^2$).

\bibliography{refs}
\bibliographystyle{iclr2023_conference}

\newpage
\appendix

\section{Experimental details}

\subsection{Memorization experiments}

\subsubsection{Setting}

To produce the results of Section \ref{sec:quantification_mem_test_collapse}, we trained simple multi-layer perceptron models with $9$ hidden layers of width $2048$. Each layer involves a batch normalization layer and uses the parameterized activation parameters (one of ReLU or sigmoid) throughout the network. To train the network, we employed SGD as optimizer with a learning rate of $0.1$ that is multiplied by $0.1$ each $40$ epochs. We further employed a Nesterov momentum of $0.9$. In total, we trained for $200$ epochs, which was sufficient to observe the neural collapse phenomenon. We ensure that the parameterization works reasonably well for all losses for a fair and realistic comparison. We further use a weight decay regularization of $0.001$. The batch size is set to $512$ for all experiments. Each assessed parameter combination has been executed $5$ times to gain statistically meaningful results.

The penultimate layer feature dimension was set to the number of classes $N$. On top of the encoding network architecture, a linear softmax classifier is attached. The entire model is optimized for four different losses: Conventional cross-entropy with degenerate target distributions, label smoothing with a default smoothing parameter of $\alpha = 0.1$, label relaxation with an imprecisiation degree of $\alpha = 0.1$ and mean squared error.

\begin{figure}[htbp]
\centering
\begin{subfigure}{0.49\columnwidth}
  \centering
  \includegraphics[width=0.8\linewidth]{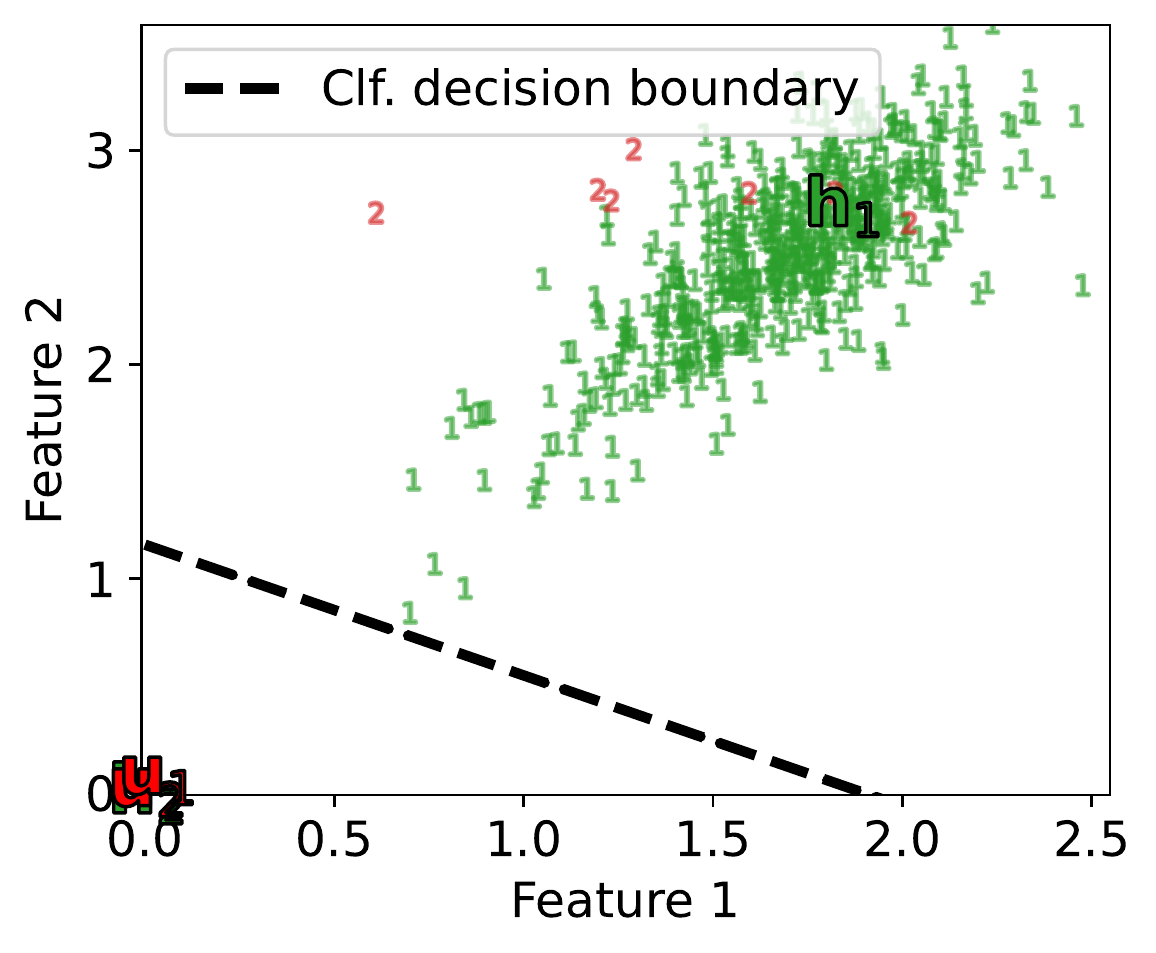}
    \caption{}
\end{subfigure}%
\begin{subfigure}{0.49\columnwidth}
  \centering
  \includegraphics[width=0.8\linewidth]{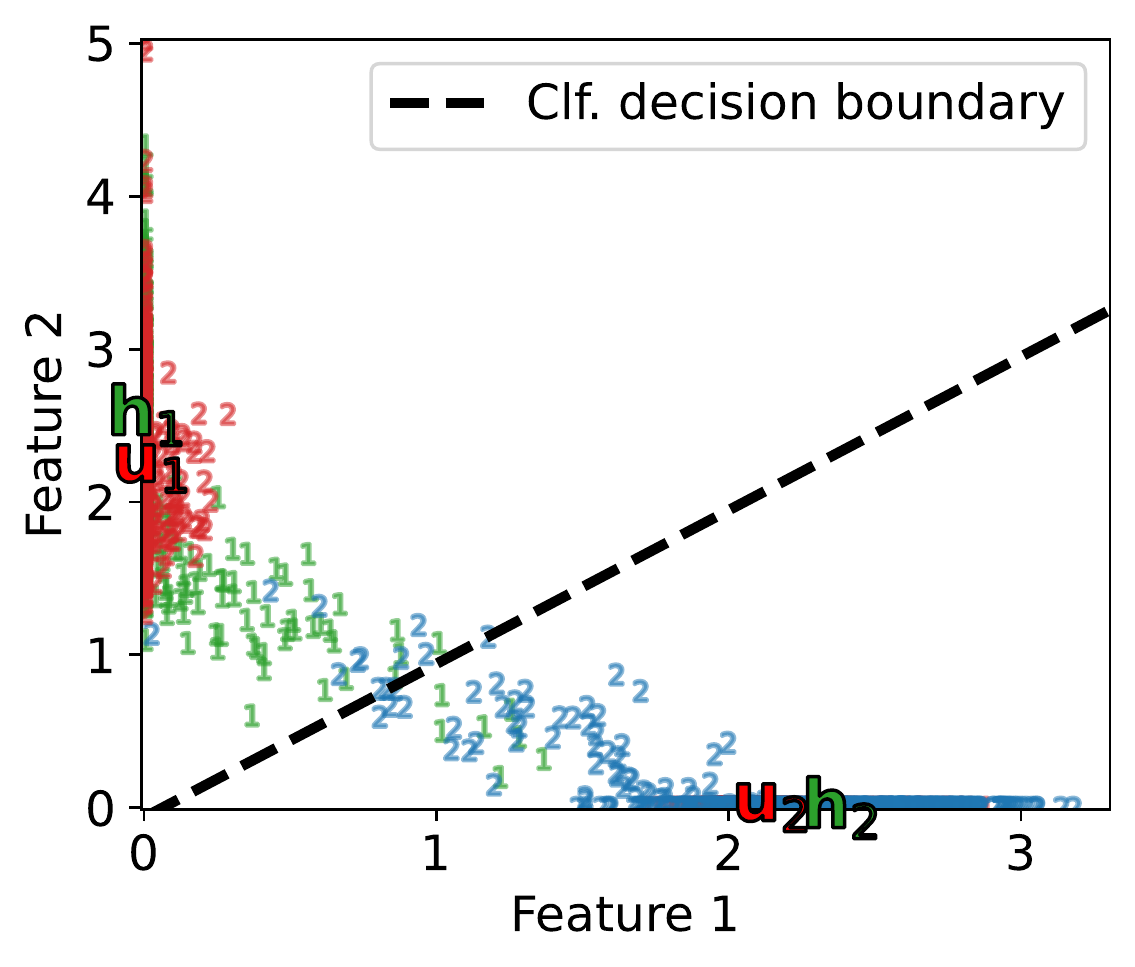}
  \caption{}
\end{subfigure}
    \caption{Exemplary penultimate layer activations (post training) of the clean and corrupted training data in the 2D feature space. \textcolor{green}{Green $1$} represent test instances of clean label $1$ data, \textcolor{blue}{blue $2$} represent clean test instances of label $2$ data, \textcolor{red}{red $1$} represent instances of training samples that were originally labeled as $1$ but were changed to label $2$, \textcolor{red}{red $2$} represent instances of training samples that were originally labeled as $2$ but were changed to label $1$.  (a) Collapse to a sub-optimal configuration, where one of the class  centroids is at the origin. (b) The class centroids are along the axes, corresponding to the optimal NC configuration of Definition \ref{def_NC}.}
    \label{fig:feature_dists}
\end{figure}

In the idealized experimental environment, we considered the datasets MNIST and CIFAR-10 as show cases. To reduce the problem complexity for the theoretical analysis, we subsampled the first $N$ classes of each dataset, all other instances were excluded. The binary case $N=2$ allows for a convenient analysis of the learned feature representations of the penultimate layer with $M = N = 2$. In case of $N=2$, cross-entropy and its derived losses did not always attain the optimal NC configuration through SGD, namely did not always align the class centroids along the axes. 
In some cases, the learned representation collapsed to one class centroid in the origin and the other one on a diagonal line in the positive quadrant in the 2D feature space. Figure \ref{fig:feature_dists} shows this case in (a) and a case the corresponds to the optimal NC configuration in (b). We filtered out the former examples, as these only infrequently occur in the $M=2$ case.

\subsubsection{Conventional label noise: Further results}

In the first label noise setting, we considered conventional label corruption, which is described in the paper. Beyond the results shown in the main part, we provide further evidence of our findings here. To this end, we repeated the experiment with different numbers of classes, namely $N \in \{3, 5, 10\}$. Figures \ref{fig:mem_exps_c3},  \ref{fig:mem_exps_c5} and  \ref{fig:mem_exps_c10} show the results. Albeit not perfect, a similar dependence can be observed for multi-class settings.

\begin{figure}[htbp]
\centering
\begin{subfigure}{\columnwidth}
  \centering
  \includegraphics[width=\linewidth]{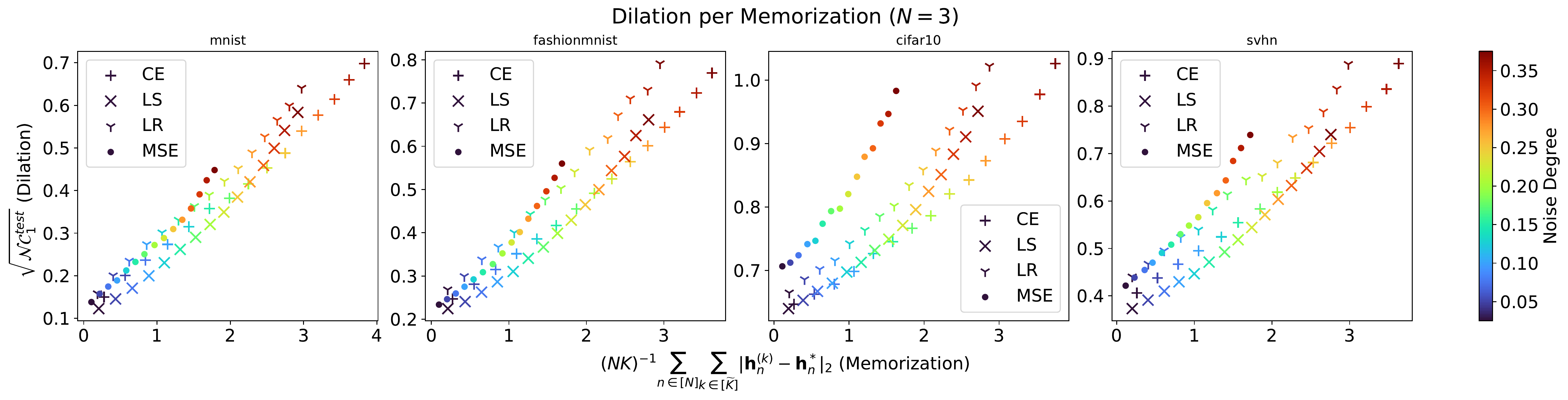}
\end{subfigure}%
\\
\begin{subfigure}{\columnwidth}
  \centering
  \includegraphics[width=\linewidth]{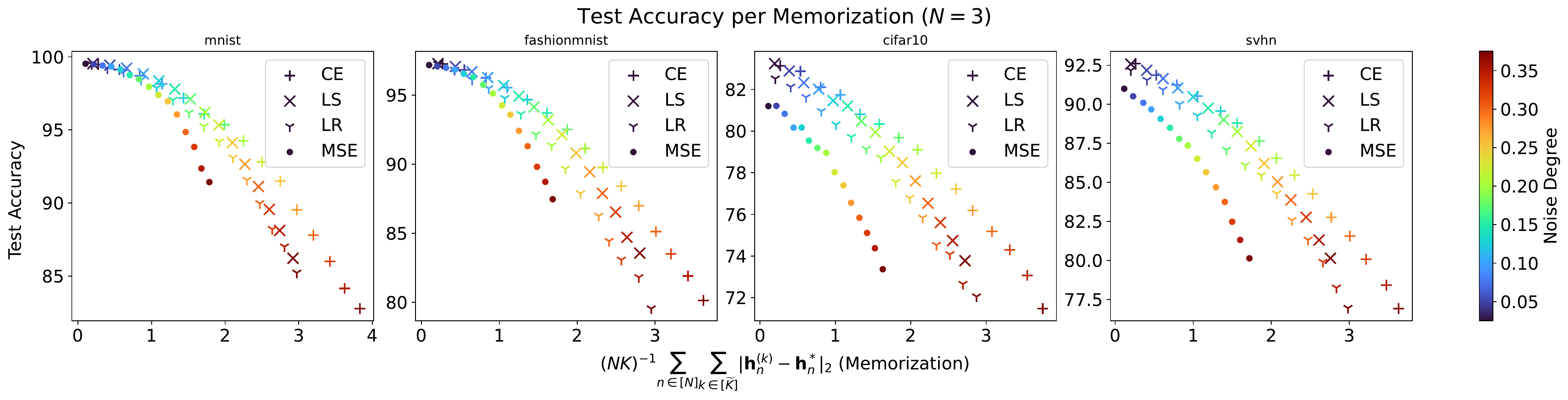}
\end{subfigure}
    \caption{Feature collapse of the test instances in terms of $\sqrt{\mathcal{NC}^{\text{test}}_1}$ per memorization and the resulting test accuracies (averaged over ten seeds) for $N=3$.}
    \label{fig:mem_exps_c3}
\end{figure}

\begin{figure}[htbp]
\centering
\begin{subfigure}{\columnwidth}
  \centering
  \includegraphics[width=\linewidth]{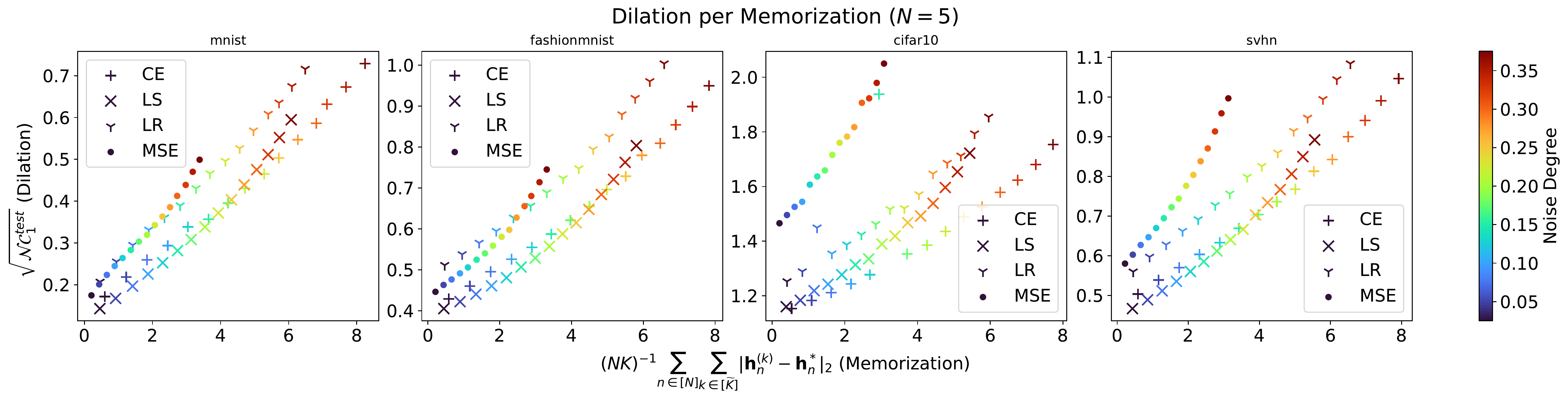}
\end{subfigure}%
\\
\begin{subfigure}{\columnwidth}
  \centering
  \includegraphics[width=\linewidth]{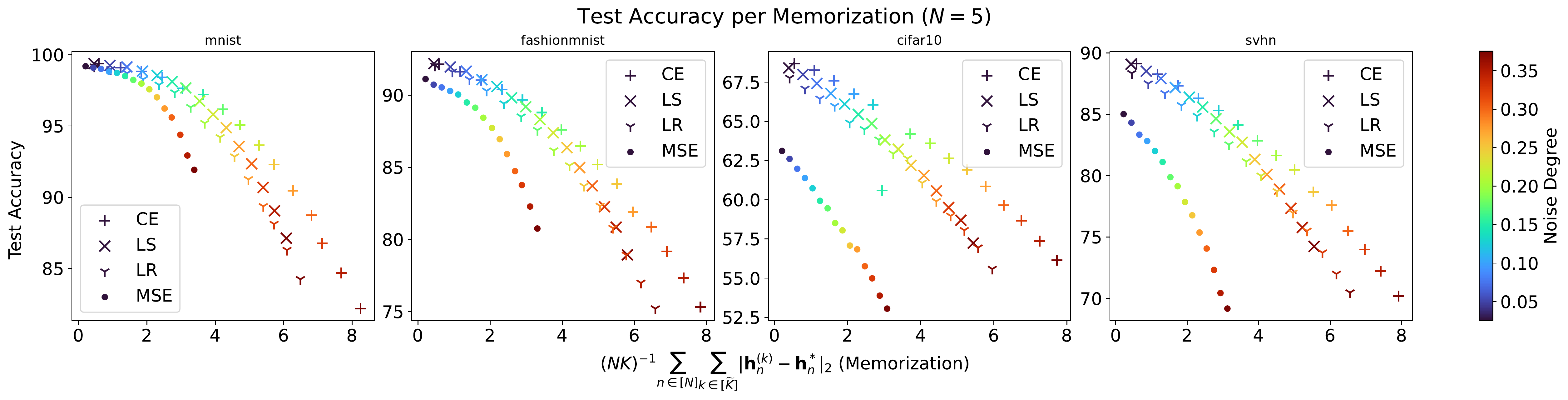}
\end{subfigure}
    \caption{Feature collapse of the test instances in terms of $\sqrt{\mathcal{NC}^{\text{test}}_1}$ per memorization and the resulting test accuracies (averaged over ten seeds) for $N=5$.}
    \label{fig:mem_exps_c5}
\end{figure}

\begin{figure}[htbp]
\centering
\begin{subfigure}{\columnwidth}
  \centering
  \includegraphics[width=\linewidth]{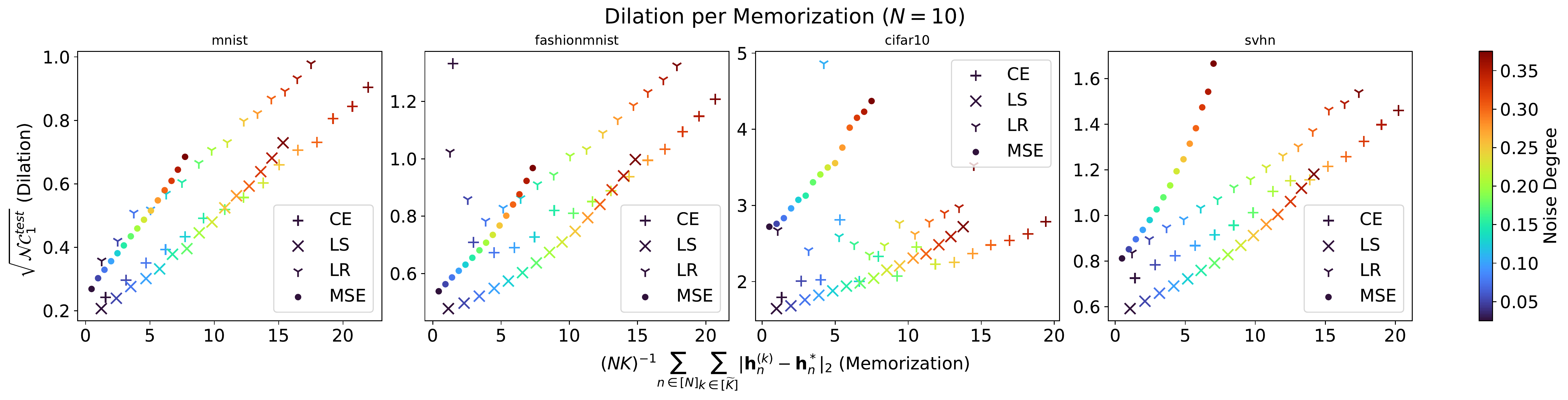}
\end{subfigure}%
\\
\begin{subfigure}{\columnwidth}
  \centering
  \includegraphics[width=\linewidth]{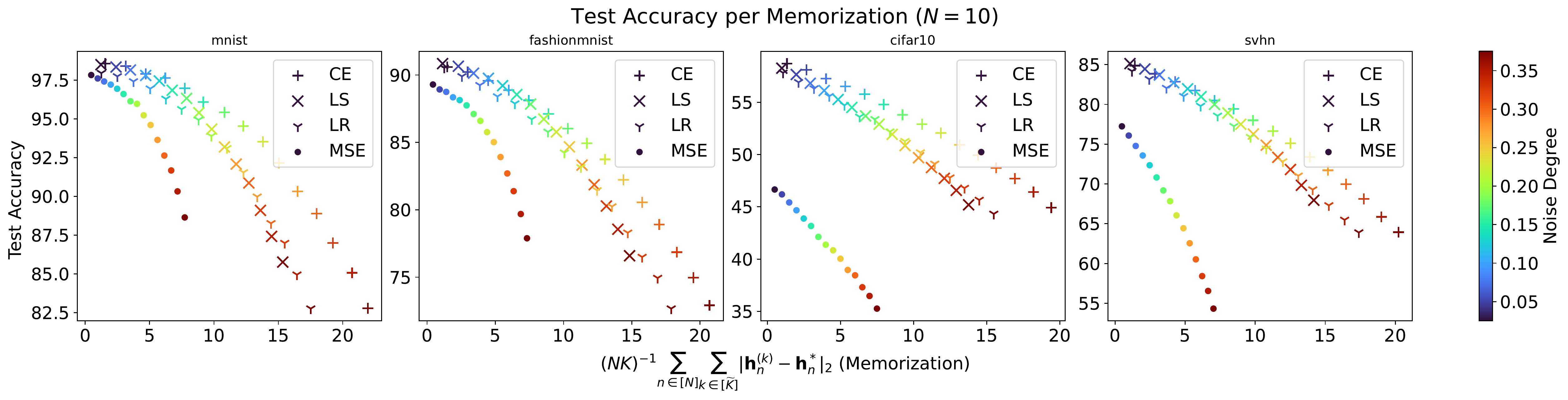}
\end{subfigure}
    \caption{Feature collapse of the test instances in terms of $\sqrt{\mathcal{NC}^{\text{test}}_1}$ per memorization and the resulting test accuracies (averaged over ten seeds) for $N=10$.}
    \label{fig:mem_exps_c10}
\end{figure}

Additionally, we show results for $M > N$ to illustrate that the correspondence also holds for higher dimensions. Figure \ref{fig:mem_exps_higher_dim} shows the resulting dilation per memorization for $N\in\set{5,10}$ on MNIST and CIFAR-10.

\begin{figure}[htbp]
\centering
\begin{subfigure}{0.45\columnwidth}
  \centering
  \includegraphics[width=\linewidth]{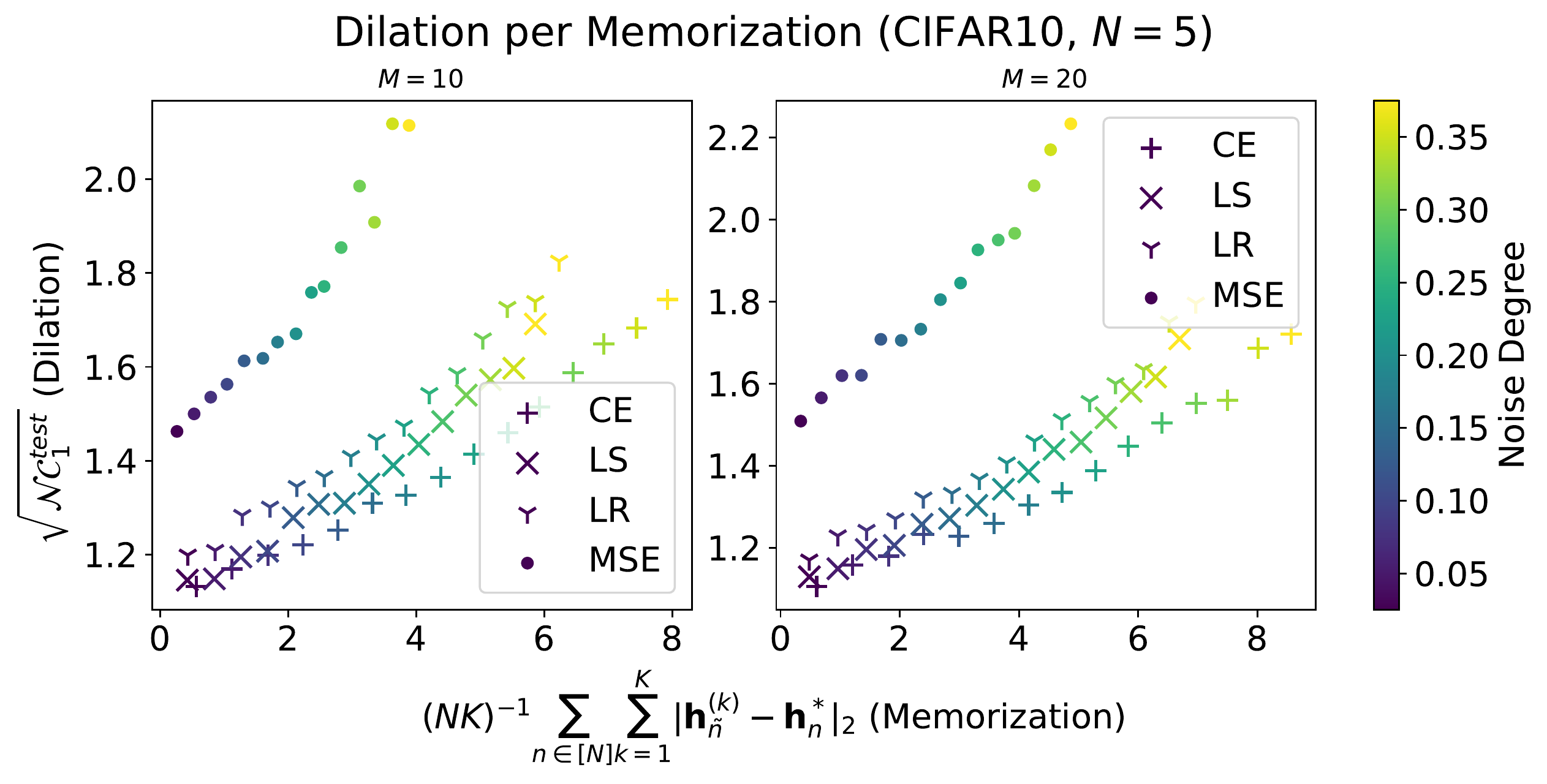}
\end{subfigure}%
\hspace{1cm}
\begin{subfigure}{0.45\columnwidth}
  \centering
  \includegraphics[width=\linewidth]{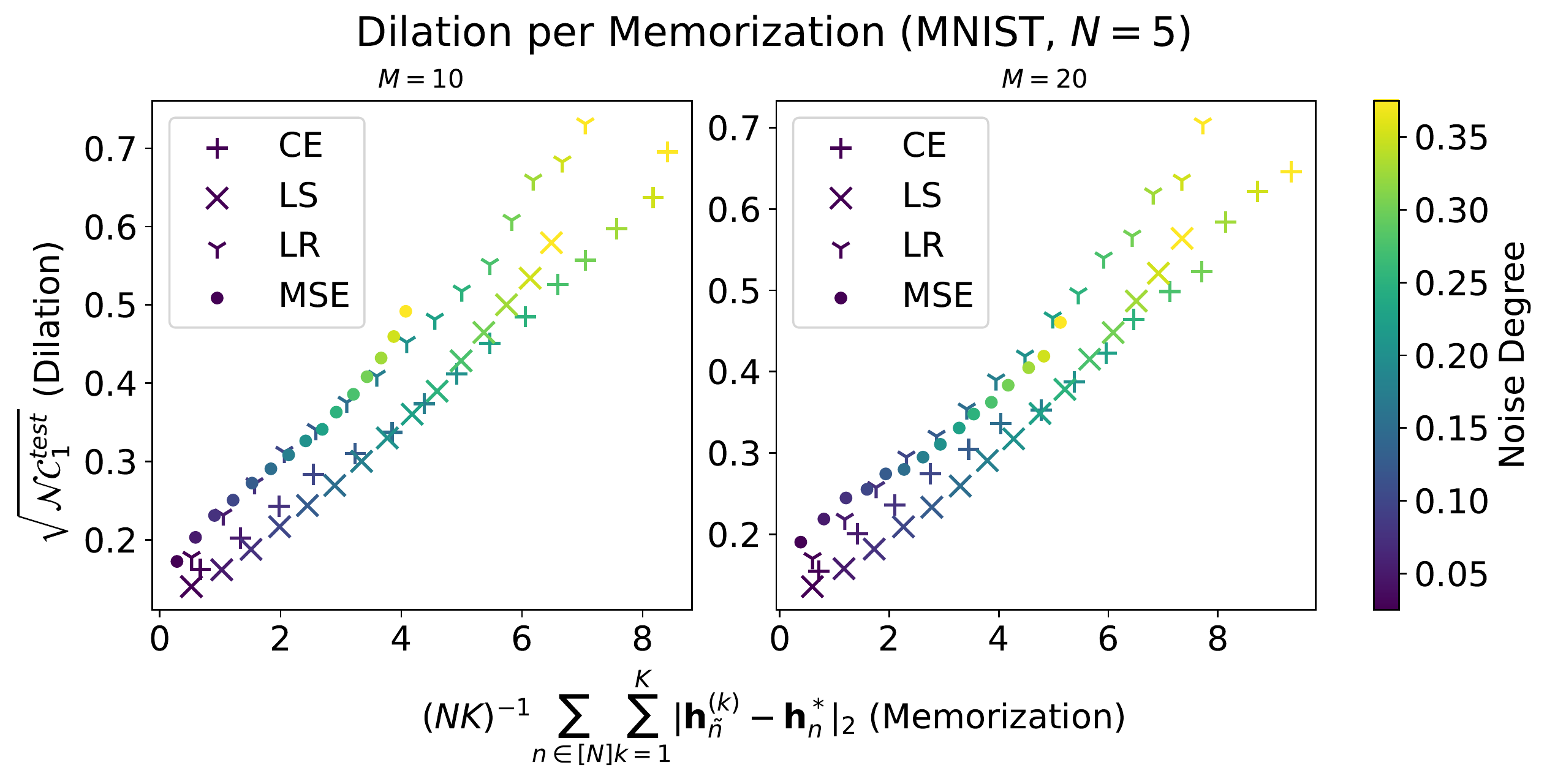}
\end{subfigure}
\\
\vspace{1cm}
\begin{subfigure}{0.45\columnwidth}
  \centering
  \includegraphics[width=\linewidth]{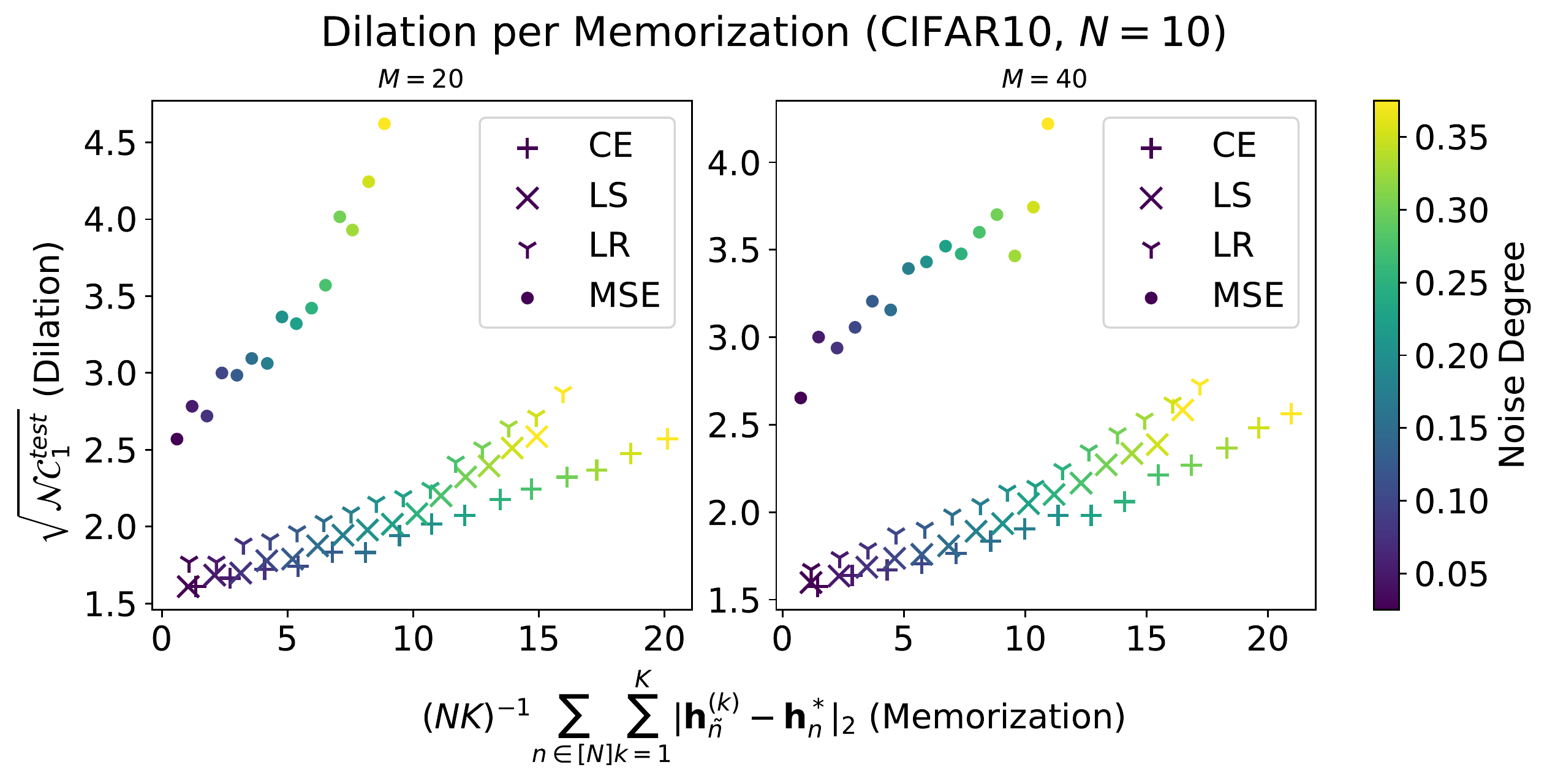}
\end{subfigure}%
\hspace{1cm}
\begin{subfigure}{0.45\columnwidth}
  \centering
  \includegraphics[width=\linewidth]{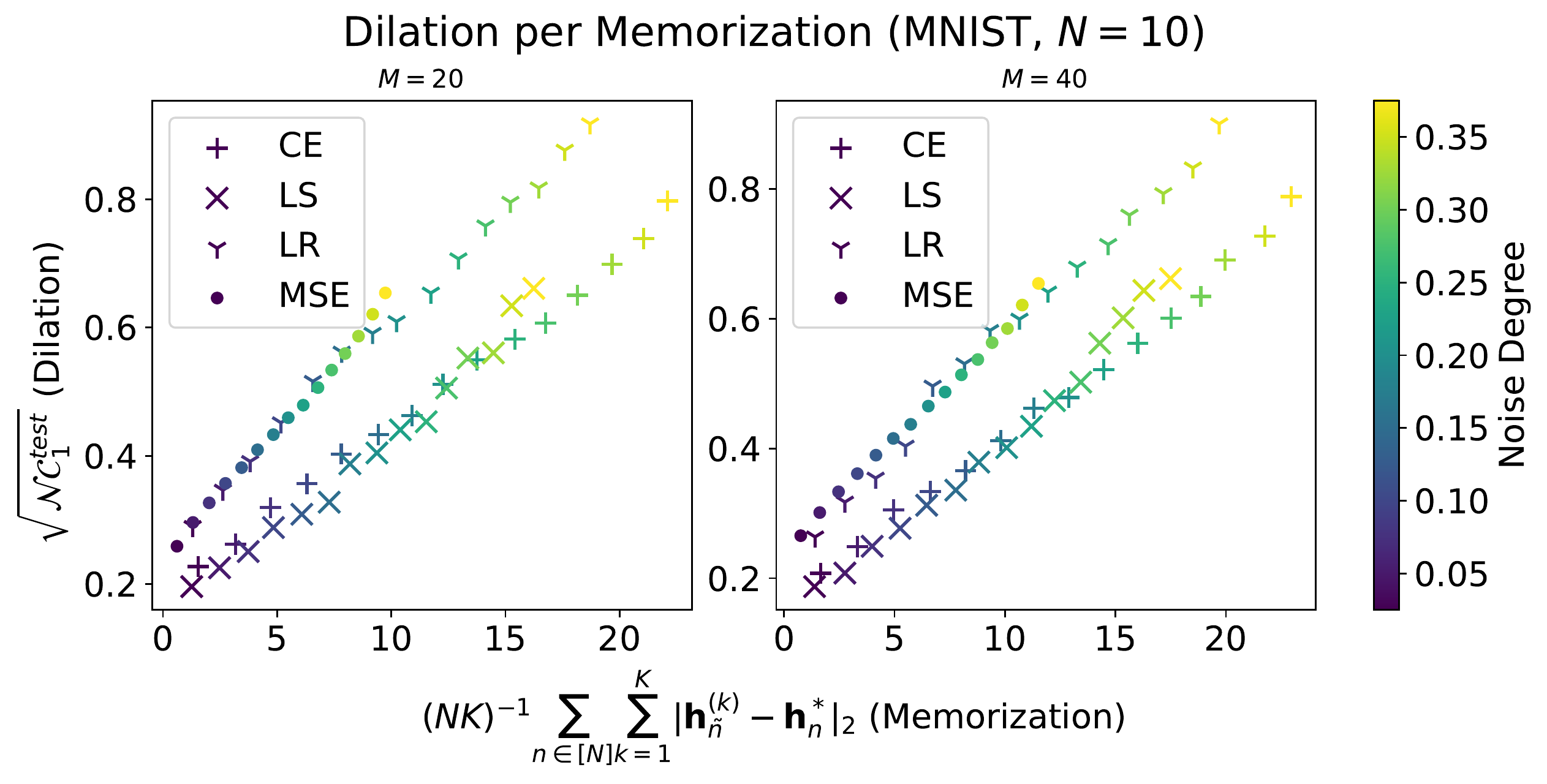}
\end{subfigure}
    \caption{Feature collapse of the test instances in terms of $\sqrt{\mathcal{NC}^{\text{test}}_1}$ per memorization for higher feature dimensions $M$ (averaged over five seeds) for $N\in \set{5,10}$.}

\label{fig:mem_exps_higher_dim}
\end{figure}

\subsubsection{Latent noise classes: Further results}

While we considered ``conventional'' label noise in the first experiment, we extend our analysis to a different form of label noise: For each original class, we split an instance fraction $\eta \in [0.025, 0.2]$ of each class apart and introduce new latent (noise) classes. Thus, the learner has again to separate these instance from their original class as it is pretended to face different classes. We consider the same basic architectural framework, but with four classes instead of two. To preserve compatibility to the previous experiments, we keep $d=N=2$. We repeated each run with $5$ different random seeds.

\begin{figure}[t]
\centering
\begin{subfigure}{0.75\columnwidth}
  \centering
  \includegraphics[width=\linewidth]{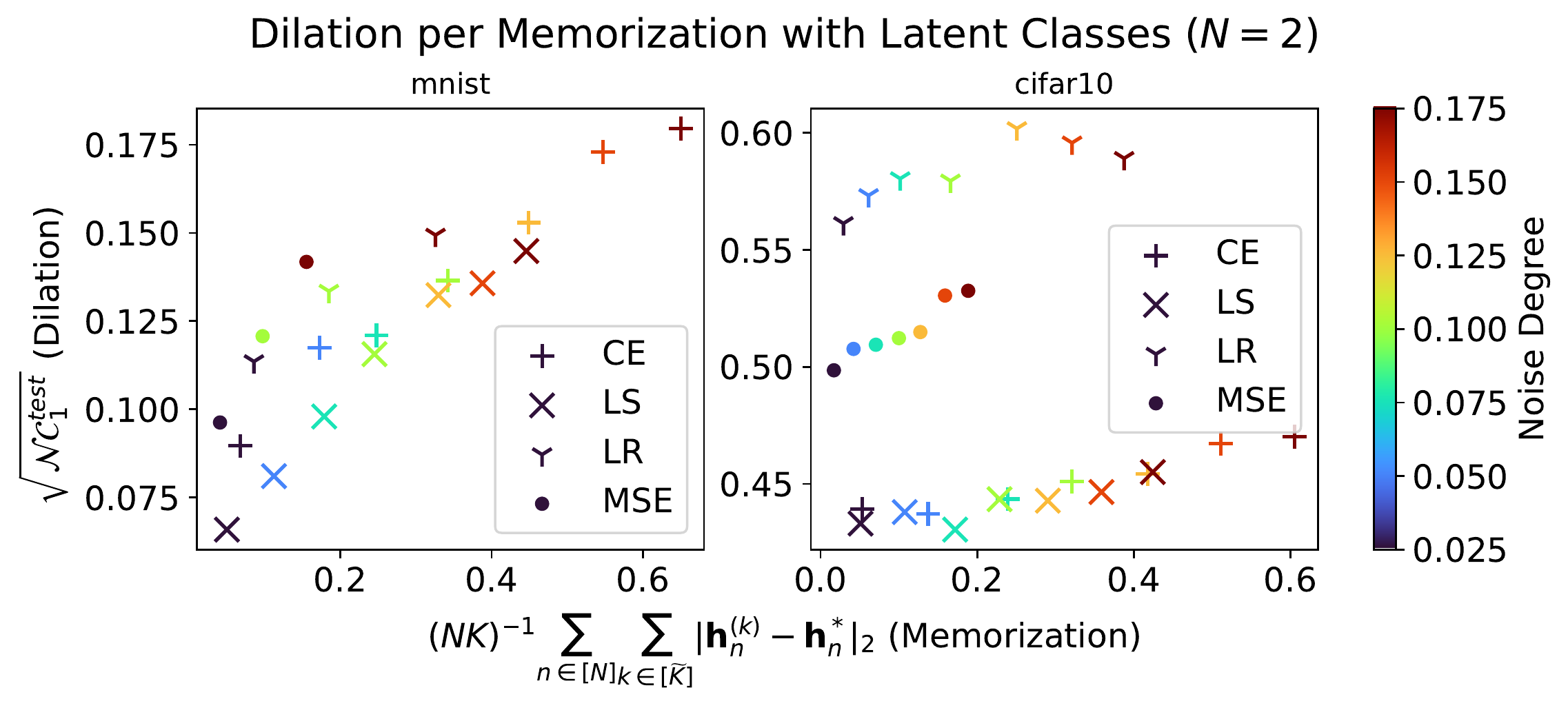}
\end{subfigure}
\\
\begin{subfigure}{0.75\columnwidth}
  \centering
  \includegraphics[width=\linewidth]{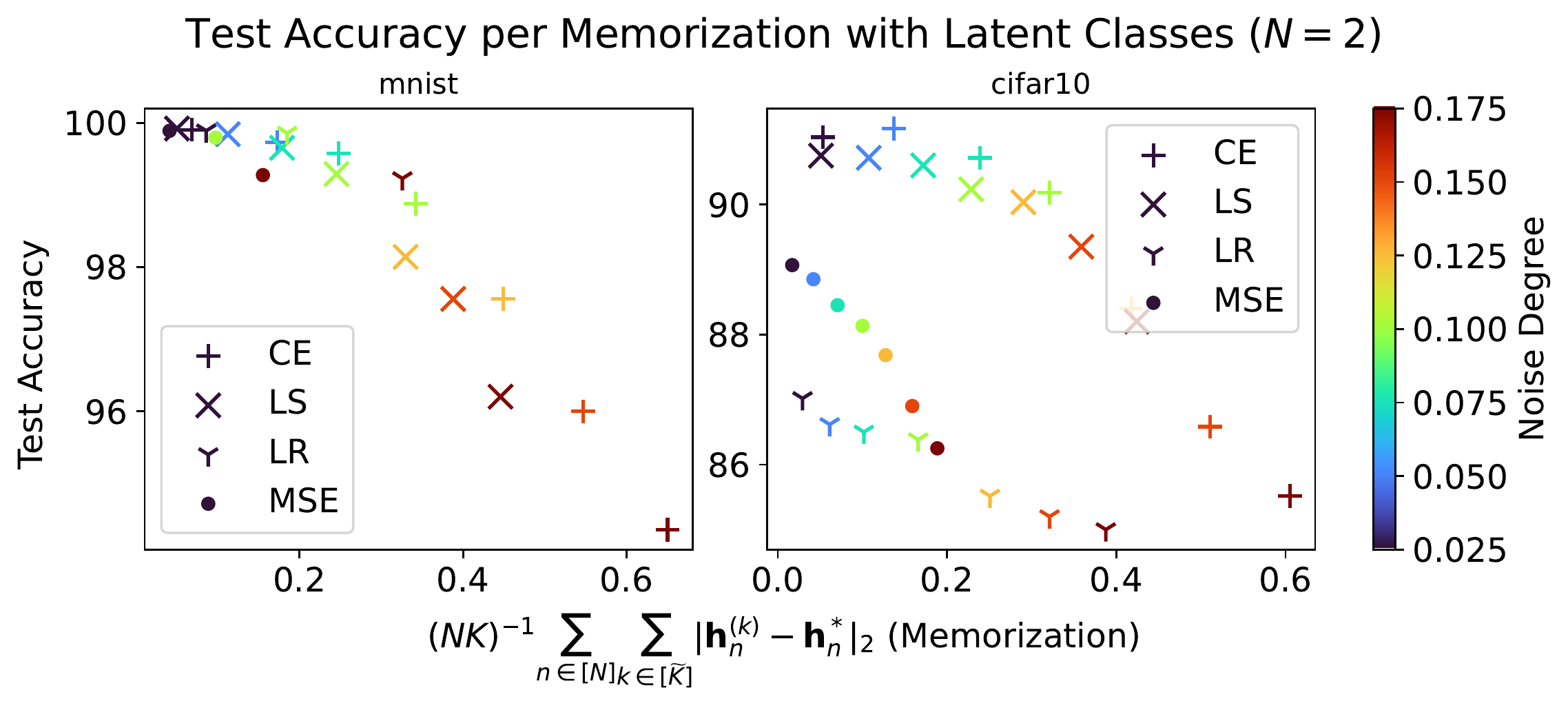}
\end{subfigure}
\caption{Dilation and test accuracy per memorization for label noise in form of latent noise classes, where we considered $N=2$ original classes. The results are averaged over $5$ random seeds.}
\label{fig:latent_classes}
\end{figure}

For this different type of label noise, the results shown in \Cref{fig:latent_classes} match the observations made before. Although the correspondence is not as clear as in the standard noise model, CE and LS are close to sharing the same curve for both datasets. Similarly, one can see a linear trend in the test collapse per memorization, which is now defined between the instances of the latent class to the test centroid of the original class. Also, we see similar trends regarding the generalization performance.

\subsection{Large-scale experiments}

\subsubsection{Setting}

Beyond the experiments in the previous section, we analyzed the neural collapse properties when training commonly used architectures, such as ResNet \cite{DBLP:conf/cvpr/HeZRS16} and VGG \cite{DBLP:journals/corr/SimonyanZ14a} models. To this end, we trained the variants ResNet18 and VGG13 on the four benchmarks MNIST, FashionMNIST, CIFAR-10 and SVHN. Here, we consider conventional label noise degrees $\eta \in \{0, 0.1, 0.2, 0.3\}$. To ensure a fair comparison, we optimized hyperparameters, such as the learning rate schedule and the smoothing and relaxation parameters $\alpha$ for LS and LR, in a Bayesian optimization using Hyperband \cite{DBLP:conf/mlsys/LiJRGBHRT20}. We tuned these parameters based on a $20 \%$ separated validation split in the no-noise case $\eta = 0$, and applied the best parameters in the noise settings with $\eta > 0$.

\begin{table}[t]
    \centering
    \begin{tabular}{ll}
        \toprule
         Parameter & Space  \\
         \midrule
         Initial learning rate & $[1e^{-5}, 0.5]$ \\
         Learning rate multiplier & $\{0.01, 0.1, 0.5\}$ \\
         Smoothing parameter $\alpha$ (LS) & $[0.01, 0.25]$ \\
         Relaxation parameter $\alpha$ (LR) & $[0.01, 0.25]$ \\
         \bottomrule
    \end{tabular}
    \caption{Search space considered within the hyperparameter optimization in the large-scale experiments.}
    \label{tab:appendix:ho_spaces}
\end{table}

Just as in the previous experiments, we used SGD as optimizer with Nesterov momentum of $0.9$, trained for $200$ epochs with a batch size of $512$. However, as opposed to the setting before, we performed a Bayesian hyperparameter optimization employing a Hyperband scheduler \cite{DBLP:conf/mlsys/LiJRGBHRT20} on a separated 20 \% validation split. To this end, we used the \texttt{skopt}\footnote{\url{https://scikit-optimize.github.io/}, BSD license} implementation and optimized for $30$ iterations. Table \ref{tab:appendix:ho_spaces} shows the considered hyperparameter space. The final model used within the evaluation was eventually trained on the complete training set (i.e., including the validation set). We repeated each run $3$ times with different seeds and report the averaged results including their standard deviations.

\subsubsection{Results}

\begin{table*}[t]
    \centering
    \resizebox{\textwidth}{!}{%
    \begin{tabular}{llcccccccc}
    \toprule
         & & \multicolumn{2}{c}{MNIST} & \multicolumn{2}{c}{FashionMNIST} & \multicolumn{2}{c}{CIFAR-10} & \multicolumn{2}{c}{SVHN} \\
         $\eta$ & Loss & ResNet18 & VGG13 & ResNet18 & VGG13 & ResNet18 & VGG13 & ResNet18 & VGG13 \\
         \midrule
         \multirow{2}{*}{0.0} & CE &  99.58 $\pm$ {\small 0.05} & 99.22 $\pm$ {\small 0.18} & 92.64 $\pm$ {\small 0.20} & 91.87 $\pm$ {\small 0.10} & 78.46 $\pm$ {\small 0.44} & 80.40 $\pm$ {\small 1.15} & 93.44 $\pm$ {\small 0.05} & 94.55 $\pm$ {\small 0.10} \\
         & LS & 99.56 $\pm$ {\small 0.01} & 99.34 $\pm$ {\small 0.08} & 92.36 $\pm$ {\small 0.13} & 91.94 $\pm$ {\small 0.56} & 78.85 $\pm$ {\small 0.29} & 81.33 $\pm$ {\small 0.38} & 93.55 $\pm$ {\small 0.21} & 92.50 $\pm$ {\small 0.50} \\
         \midrule
         \multirow{2}{*}{0.1} & CE & 98.32 $\pm$ {\small 0.11} & 97.01 $\pm$ {\small 0.28} & 89.35 $\pm$ {1.65} & 89.81 $\pm$ {\small 0.30} & 70.86 $\pm$ {\small 1.54} & 76.12 $\pm$ {\small 1.58} & 88.89 $\pm$ {\small 0.79} & 90.51 $\pm$ {\small 0.24} \\
         & LS & 98.51 $\pm$ {\small 0.14} & 98.16 $\pm$ {\small 0.03} & 89.62 $\pm$ {\small 0.53} & 90.41 $\pm$ {0.52} & 71.82 $\pm$ {\small 1.24} & 76.13 $\pm$ {\small 1.74} & 89.68 $\pm$ {\small 0.40} & 90.60 $\pm$ {\small 0.40} \\
         \midrule
         \multirow{2}{*}{0.2} & CE & 96.34 $\pm$ {\small 0.20} & 92.72 $\pm$ {\small 1.21} & 85.20 $\pm$ {\small 2.67} & 85.39 $\pm$ {\small 0.37} & 61.39 $\pm$ {\small 1.30} & 70.09 $\pm$ {\small 0.46} & 83.40 $\pm$ {\small 2.19} & 86.03 $\pm$ {\small 0.03} \\
         & LS & 96.48 $\pm$ {\small 0.13} & 95.00 $\pm$ {\small 0.04} & 86.27 $\pm$ {\small 1.24} & 87.54 $\pm$ {\small 0.68} & 63.69 $\pm$ {\small 3.06} & 70.97 $\pm$ {\small 1.58} & 85.40 $\pm$ {\small 0.49} & 86.81 $\pm$ {\small 0.70} \\
         \midrule
         \multirow{2}{*}{0.3} & CE & 91.82 $\pm$ {\small 0.22} & 87.87 $\pm$ {\small 0.64} & 79.91 $\pm$ {\small 3.26} & 80.01 $\pm$ {\small 0.65} & 52.52 $\pm$ {\small 0.05} & 63.22 $\pm$ {\small 1.26} & 77.10 $\pm$ {\small 1.55} & 59.71 $\pm$ {\small 28.37} \\
         & LS & 92.05 $\pm$ {\small 0.31} & 88.07 $\pm$ {\small 0.42} & 81.18 $\pm$ {\small 2.41} & 82.59 $\pm$ {\small 1.25} & 55.31 $\pm$ {\small 2.09} & 63.63 $\pm$ {\small 2.88} & 79.62 $\pm$ {\small 0.44} & 80.71 $\pm$ {\small 1.56} \\
         \bottomrule
    \end{tabular}
    }
    \caption{Generalization performances and their standard deviations in terms of test accuracies for different label noise degrees $\eta$ (average over $3$ seeds).}
    \label{tab:large_scale_exps}
\end{table*}

Table \ref{tab:large_scale_exps} shows the resulting generalization performances as an average over $3$ seeds. As can be seen, label smoothing consistently improves over cross-entropy, confirming both the empirical and theoretical observations as presented before. These results suggests that label smoothing is particularly appealing in case of label noise.

\subsection{Technical infrastructure}

To realize the experiments, we proceeded from the official code base of \cite{DBLP:journals/corr/abs-2105-02375}\footnote{\url{https://github.com/tding1/Neural-Collapse}} and augmented it by further baselines, models and our evaluation metrics. This implementation leverages \texttt{PyTorch}\footnote{\url{https://pytorch.org/}, BSD license} as deep learning framework and obtains data and models from \texttt{torchvision}\footnote{\url{https://pytorch.org/vision/}, BSD license}. To execute the runs, we used Nvidia GPU accelerators (1080/2080 Ti, Titan RTX) in a modern cluster environment. Our code is publicly available at \url{https://github.com/julilien/MemorizationDilation}.

\section{Theoretical supports for Theorem \ref{theorem:standard_model}}
In this appendix we introduce the theoretical supports for our theorem on the layer-peeled model, i.e. Theorem \ref{theorem:standard_model}. Before going to the proof, we will shortly discuss in Subsection \ref{subsection:discussion_standard_result} about how this finding differs from several related ones while still entailing the NC properties NC1-NC3. Then we will introduce some auxiliary results that are helpful for the proof in Subsections \ref{Sec: Reformulation_loss} and \ref{subsection:technical_lemmas}, and the proof itself in Subsection \ref{subsection:proof_standard_model}.

Note that the concurrent work \cite{all_losses_equal_2022} studies a very similar problem as \Cref{Problem:standard}, where the loss function $L_\alpha$ is replaced by a more general one that satisfies the so-called \emph{contrastive property}. This covers both CE and LS loss, and hence the proof in this work is quite similar to ours. The main difference between the two models is namely the positivity constraint on the features, which leads to the main technical difference in the proof. 

Shortly speaking, the ultimate goal in the proof in \cite{all_losses_equal_2022} (or several works of the same type, e.g. \cite{DBLP:journals/corr/abs-2105-02375}) and ours is to show that the loss is lower bounded by some constant, and equality occurs only if the NC configuration is satisfied. However, if the minimization problem is unconstrained, one just needs to consider critical points, which might have a nice form. Although this is a clever trick, it cannot (at least not directly) be applied to the constrained problem, because the first-order optimality condition does not give as much information as in the unconstrained case (since it does not involve a simple equality form). 

To overcome this, we will find a lower bound of the loss function in terms of $\norm{\vec{W}}$ and $\norm{\vec{H}}$, instead of the variables $\vec{W}$ and $\vec{H}$ themselves. This is possible in some certain region of the set of all feasible values of $\vec{W}$ and $\vec{H}$ (case (c) in Step 2 in our proof), while for other regions, we will show that either the minimizer does not belong to the region, or when it does, it must be also at NC configuration. Finally, the lower bound in terms of $\norm{\vec{W}}$ and $\norm{\vec{H}}$ is easier to deal with than the original one in terms of $\vec{W}$ and $\vec{H}$. In fact, we will see that despite the positivity of $\norm{\vec{W}}$ and $\norm{\vec{H}}$, one could consider the mentioned lower bound as a function of two variables that can take any value in $\R$. 

\subsection{Discussion about the result}\label{subsection:discussion_standard_result}

The configurations defined in \Cref{def_NC} above differ from the ones specified in other works, e.g. \citet{Fang2021ExploringDN,DBLP:journals/corr/abs-2105-02375,all_losses_equal_2022} and describe more precisely the empirically observed NC phenomena. Indeed, it has been shown in those papers that the minimizers of (\ref{Problem:standard}) without positivity constraint on the features must satisfy the following conditions
\begin{enumerate}[label=(\roman*$'$)]
  \item The feature representations $\vec{h}_n^{(k)}$ within every class $n \in [N]$ are equal for all $k\in [K]$, and thus equal to their class mean $\vec{h}_n := \frac{1}{K}\sum_{k=1}^K \vec{h}_n^{(k)}$.
  \item The class means $\{\vec{h}_n\}_{n=1}^N$ have equal norms and form an $N$-simplex equiangular tight frame (ETF) up to some rescaling. 
  \item The weight matrix $\vec{W}$ satisfies  $\vec{w}_n = C\vec{h}_n$ 
  for some constant $C>0$.
\end{enumerate}

Although the configuration defined by (i')-(iii') is different from the one defined in Def. \ref{def_NC}, both entail the limit of the NC1-NC3 properties. Indeed, it is straightforward to observe the relation between NC1 and (i) or (i'). Moreover, the limit of NC2 is directly implied from (ii'), namely because the global mean $\vec{h} = \frac{1}{N}\sum_{n=1}^N \vec{h}_n$ must lie at the origin, i.e. $\vec{h}=0$, while 
\begin{align*}
    \norm{\vec{h}_m} = \norm{\vec{h}_n} \quad \text{and} \quad \sprod{\frac{\vec{h}_m}{\norm{\vec{h}_m}}, \frac{\vec{h}_n}{\norm{\vec{h}_n}}} = -\frac{1}{N-1} \quad \text{for any } m\neq n 
\end{align*}
hold as simple properties of a simplex ETF. The limit of NC3 follows also directly from the duality condition (iii') and the above observation that $\vec{h}=0$. On the other hand, it is not as straightforward, but also not difficult to see the connection between the conditions (ii) and (iii) in Def. \ref{def_NC}. Indeed, from (ii) it follows
\begin{align*}
    \norm{\vec{h}_m-\vec{h}}^2 &= \norm{\vec{h}_m}^2 - 2\sprod{\vec{h}_m,\vec{h}} + \norm{\vec{h}}^2 \\
    &= \norm{\vec{h}_m}^2 - \frac{2}{N}\norm{\vec{h}_m}^2 + \frac{1}{N}\norm{\vec{h}_m}^2\\
    &= \frac{N-1}{N} \norm{\vec{h}_m}^2,
\end{align*}
which combining with $\norm{\vec{h}_m} = \norm{\vec{h}_n}$ gives $\norm{\vec{h}_m-\vec{h}} = \norm{\vec{h}_n-\vec{h}}$ for any $m \neq n$. Also, 
\begin{align*}
    \sprod{\vec{h}_m-\vec{h},\vec{h}_n-\vec{h}} &= -\sprod{\vec{h},\vec{h}_n}-\sprod{\vec{h},\vec{h}_m} + \norm{\vec{h}}^2\\
    &=\frac{-2}{N}\norm{\vec{h}_n}^2+ \frac{1}{N} \norm{\vec{h}_n}^2\\
    &= \frac{-1}{N}\norm{\vec{h}_n}^2,
\end{align*}
which combining with the above finding shows 
\begin{align*}
    \sprod{\frac{\vec{h}_m-\vec{h}}{\norm{\vec{h}_n-\vec{h}}_2}, \frac{\vec{h}_m-\vec{h}}{\norm{\vec{h}_n-\vec{h}}_2}} = - \frac{1}{N-1}.
\end{align*}
Finally, the duality condition (iii) in Def. \ref{def_NC} involves the projection $P_{\vec{h}^\perp} \vec{h}_n$, which is the same as $\vec{h}_n -\vec{h}$. To see this, observe that $\vec{h}_n- (\vec{h}_n-\vec{h}) = \vec{h} \perp \vec{h}^\perp$ and $\sprod{\vec{h}_n-\vec{h}, \vec{h}} = \frac{1}{N}\norm{\vec{h}_n} - \frac{1}{N^2} \sum_{i=1}^N \norm{\vec{h}_i}^2 = 0.$

Hence, the condition (iii) simply means that $\vec{w}_n$ is proportional to $\vec{h}_n-\vec{h}$, which corresponds with the limit of the property NC3. 

The main difference between our definition of NC configuration (i.e. Def. \ref{def_NC}) and the one described by the conditions (i')-(iii') above is that we require the \emph{centralized} class means $\set{\vec{h}_n-\vec{h}}_{n=1}^N$ to form a simplex ETF, not the class means $\set{\vec{h}_n}_{n=1}^N$ themselves. Similarly, the duality in our definition involves the weights $\vec{w}_n$ and the \emph{centralized} class means $\vec{h}_n -\vec{h}$, and not directly the class means or class features. Concerning the class means $\set{\vec{h}_n}_{n=1}^N$, we require them in (ii) to be an equinorm orthogonal system, which differs from a simplex ETF. In practice, if ReLU operation is applied, the features must be positive and hence the class means cannot form or approximate a simplex ETF, which must center at the origin. 

In this sense, the NC configuration defined in Def. \ref{def_NC} can be seen as capturing more closely the NC phenomena in practice. Meanwhile, the usual configuration (i')-(iii') is obtained by a translation by the global mean $\vec{h}$ from our NC configuration. This translation may drop several interesting properties of the configuration, for example the property that the class means $\vec{h}_n$ are orthogonal and therefore (as they are positive) have separate supports (i.e. the indices of the nonzero entries in $\vec{h}_m$ and $\vec{h}_n$ do not overlap). 

Notably, our NC configuration defined in Definition \ref{def_NC} and the \emph{orthogonal frame} 
configuration in Theorem 3.1 in \citet{Extended_unconstrained_features_model} appear to be similar, but have certain differences. Despite having an equivalent description for $\vec{H}$, our work considers positive features, which requires the feature vectors of different classes to have separate supports, i.e. their entries are supported on disjoint dimensions. Moreover, the weights $\vec{w}_n$ in our approach, as stated in (iii) in Def. \ref{def_NC}, are not proportional to the class means $\vec{h}_n$ but to the centralized ones $\vec{h}_n-\vec{h}$, and hence form a simplex ETF and not an equinorm orthogonal system. Moreover, even in the presence of bias terms, the resulting configuration in our paper remains unchanged, which is different to Theorem 3.2 in \citet{Extended_unconstrained_features_model}. Note that all these differences come from the fact that we consider the CE or LS loss, while the authors of \citet{Extended_unconstrained_features_model} study the MSE loss.

\subsection{Reformulation of the LS empirical risk} \label{Sec: Reformulation_loss}

Given a smoothing parameter $\alpha \in [0,1)$, we will write the LS empirical risk introduced in Section \ref{Section:standard_model} in more details,
\begin{align*}
    L_\alpha(\vec{W},\vec{H}) &= \frac{1}{NK} \sum_{k=1}^K \sum_{n=1}^N \ell_\alpha\Big(\vec{W}, \vec{h}_n^{(k)}, \vec{y}_n^{(\alpha)} \Big)\\
    &= \frac{1}{NK} \sum_{k=1}^K \sum_{n=1}^N 
    \Bigg[(1-\frac{N-1}{N}\alpha) \log \Big( \sum_{i=1}^N e^{\sprod{\vec{w}_i-\vec{w}_n,\vec{h}_n^{(k)}}}\Big) \\
    &\quad + \sum_{\substack{m=1\\m\neq n}}^N  \frac{\alpha}{N} \log \Big( \sum_{i=1}^N e^{\sprod{\vec{w}_i-\vec{w}_m,\vec{h}_n^{(k)}}}\Big) \Bigg]\\
    &=\frac{1}{NK} \sum_{k=1}^K \sum_{n=1}^N 
    \Bigg[(1-\frac{N-1}{N}\alpha) \log \Big( \sum_{i=1}^N e^{\sprod{\vec{w}_i-\vec{w}_n,\vec{h}_n^{(k)}}}\Big) \\
    &\quad + \sum_{\substack{m=1\\m\neq n}}^N  \frac{\alpha}{N} \log \Big( e^{\sprod{\vec{w}_n-\vec{w}_m,\vec{h}_n^{(k)}}}\sum_{i=1}^N e^{\sprod{\vec{w}_i-\vec{w}_n,\vec{h}_n^{(k)}}}\Big) \Bigg]\\
    &= \frac{1}{NK} \sum_{k=1}^K \sum_{n=1}^N 
    \Bigg[(1-\frac{N-1}{N}\alpha) \log \Big( \sum_{i=1}^N e^{\sprod{\vec{w}_i-\vec{w}_n,\vec{h}_n^{(k)}}}\Big)  \\
    &\quad + \sum_{\substack{m=1\\m\neq n}}^N  \frac{\alpha}{N} \log \Big( \sum_{i=1}^N e^{\sprod{\vec{w}_i-\vec{w}_n,\vec{h}_n^{(k)}}}\Big) + \sum_{\substack{m=1\\m\neq n}}^N  \frac{\alpha}{N}  \sprod{\vec{w}_n-\vec{w}_m,\vec{h}_n^{(k)}} \Bigg]
\end{align*}
Hence,
\begin{align*}
    L_\alpha(\vec{W},\vec{H}) 
    &= \frac{1}{NK} \sum_{k=1}^K \sum_{n=1}^N 
    \Bigg[(1-\frac{N-1}{N}\alpha) \log \Big( \sum_{i=1}^N e^{\sprod{\vec{w}_i-\vec{w}_n,\vec{h}_n^{(k)}}}\Big)  \\
    &\quad +   \frac{N-1}{N}\alpha \log \Big( \sum_{i=1}^N e^{\sprod{\vec{w}_i-\vec{w}_n,\vec{h}_n^{(k)}}}\Big) + \sum_{\substack{m=1\\m\neq n}}^N  \frac{\alpha}{N}  \sprod{\vec{w}_n-\vec{w}_m,\vec{h}_n^{(k)}} \Bigg]\\
    &= \frac{1}{NK} \sum_{k=1}^K \sum_{n=1}^N 
    \Bigg[ \log \Big( \sum_{i=1}^N e^{\sprod{\vec{w}_i-\vec{w}_n,\vec{h}_n^{(k)}}}\Big)  + \sum_{\substack{m=1\\m\neq n}}^N  \frac{\alpha}{N}  \sprod{\vec{w}_n-\vec{w}_m,\vec{h}_n^{(k)}} \Bigg]\\
    &= \frac{1}{NK} \sum_{k=1}^K \sum_{n=1}^N 
    \Bigg[ \log \Big( 1+\sum_{\substack{m=1\\m\neq n}}^N e^{\sprod{\vec{w}_m-\vec{w}_n,\vec{h}_n^{(k)}}}\Big)  - \sum_{\substack{m=1\\m\neq n}}^N  \frac{\alpha}{N}  \sprod{\vec{w}_m-\vec{w}_n,\vec{h}_n^{(k)}} \Bigg].
\end{align*}
Shortly speaking, this differs from the conventional CE loss just by an additional bilinear term $\frac{\alpha}{N}\frac{1}{NK}\sum_{k=1}^K \sum_{n=1}^N \sum_{m \neq n} \sprod{\vec{w}_m-\vec{w}_n,\vec{h}_n^{(k)}}$. 

\subsection{Technical lemmata}\label{subsection:technical_lemmas}
\begin{lemma}\label{Lemma:CauchySchwarz}
We define 
\begin{align*}
    P(\vec{W},\vec{H}) := \frac{1}{KN(N-1)}\sum_{k=1}^K \sum_{n=1}^N \sum_{\substack{m=1\\m\neq n}}^N\sprod{\vec{w}_m-\vec{w}_n,\vec{h}_n^{(k)}}.
\end{align*}
Then under the condition $\vec{H}\geq 0$ it holds
\begin{align}\label{ineq:Cauchyschwarz}
    P(\vec{W},\vec{H}) \geq -\frac{1}{\sqrt{KN(N-1)}} \norm{\vec{W}} \norm{\vec{H}}.
\end{align}
 
The inequality (\ref{ineq:Cauchyschwarz}) becomes an equality if and only if the following conditions hold simultaneously
\begin{align}\label{eq1:4}
    \sum_{n=1}^N \vec{w}_n = 0
\end{align}
\begin{align}\label{eq1:5}
    \sprod{\vec{h}_n^{(k)}, \vec{h}_m^{(k)}} = 0 \quad \text{for all } m,n \in [N], k\in [K], m\neq n,
\end{align}
\begin{align}\label{eq1:6}
    \norm{\vec{h}_n^{(k)}}  \quad \text{is independent of $n,k$},
\end{align}
\begin{align}\label{eq1:7}
    \vec{w}_m-\vec{w}_n = c'(\vec{h}_m^{(k)} - \vec{h}_n^{(k)}) \quad \text{for some } c'>0 \text{ not depending on $m,n,k$}.
\end{align}

\end{lemma}

\begin{proof}
Using the Cauchy-Schwarz inequality we get
    \begin{align*}
        P(W,H) &:= \frac{1}{KN(N-1)}\sum_{k=1}^K \sum_{n=1}^N \sum_{\substack{m=1\\m\neq n}}^N\sprod{\vec{w}_m-\vec{w}_n,\vec{h}_n^{(k)}} \\ 
        &= \frac{1}{KN(N-1)}\sum_{k=1}^K \sum_{n=1}^N \sum_{m=n+1}^N \sprod{\vec{w}_m-\vec{w}_n,\vec{h}_n^{(k)} - \vec{h}_m^{(k)}} \\
        &\geq -\frac{1}{KN(N-1)}\sum_{k=1}^K \sqrt{\Big( \sum_{n=1}^N \sum_{m=n+1}^N \norm{\vec{w}_n-\vec{w}_m}^2\Big) \Big( \sum_{n=1}^N \sum_{m=n+1}^N \norm{\vec{h}_n^{(k)}-\vec{h}_m^{(k)}}^2\Big)}\\
        &= -\frac{1}{KN(N-1)}\underbrace{\sqrt{\sum_{n=1}^N \sum_{m=n+1}^N \norm{\vec{w}_n-\vec{w}_m}^2}}_{=:P_1} \underbrace{\sum_{k=1}^K\sqrt{\sum_{n=1}^N \sum_{m=n+1}^N \norm{\vec{h}_n^{(k)}-\vec{h}_m^{(k)}}^2}}_{=:P_2}.
    \end{align*}
    Further application of Cauchy-Schwarz inequality yields
    \begin{align*}
        P_1 &= \sqrt{\sum_{n} \sum_{m=n+1}^N \norm{\vec{w}_n-\vec{w}_m}^2} \\
        &= \sqrt{N\sum_{n=1}^N \norm{\vec{w}_n}^2 - \norm{\sum_{n=1}^N \vec{w}_n}^2} \leq \sqrt{N}\norm{\vec{W}}
    \end{align*}
    and 
    \begin{align*}
        P_2 &= \sum_{k=1}^K\sqrt{\Big( \sum_{n=1}^N \sum_{m =n+1}^N \norm{\vec{h}_n^{(k)}-\vec{h}_m^{(k)}}^2\Big)} \\
        &= \sum_{k=1}^K\sqrt{ (N-1) \sum_{n=1}^N \norm{\vec{h}_n^{(k)}}^2 - \sum_{n=1}^N\sum_{m=n+1}^N\sprod{\vec{h}_n^{(k)}, \vec{h}_m^{(k)}}}\\
        &\leq \sqrt{N-1}\sum_{k=1}^K\sum_{n=1}^N \norm{\vec{h}_n^{(k)}} \\
        &\leq \sqrt{KN(N-1)} \sqrt{\sum_{k=1}^K\sum_{n=1}^N \norm{\vec{h}_n^{(k)}}^2} = \sqrt{KN(N-1)}\norm{\vec{H}}
    \end{align*}
    Therefore
    \begin{align*}
        P(\vec{W},\vec{H}) \geq -\frac{1}{KN(N-1)}P_1P_2 \geq -\frac{1}{\sqrt{K(N-1)}} \norm{\vec{W}} \norm{\vec{H}}.
    \end{align*}
    
    This becomes an equality if and only if 
    \begin{itemize}
        \item The upper bound on $P_1$ becomes equality, i.e.
        \begin{align*}
            \sum_{n=1}^N \vec{w}_n = 0
        \end{align*}
        
        \item The upper bound on $P_2$ becomes equality, i.e.
        \begin{align*}
            \sprod{\vec{h}_n^{(k)}, \vec{h}_m^{(k)}} = 0 \quad \text{for all } m,n \in [N], k\in [K], m\neq n,
        \end{align*}
        
        \begin{align*}
            \norm{\vec{h}_n^{(k)}}  \quad \text{is independent of $n,k$}.
        \end{align*}
        
        \item The estimate $P \geq -\frac{1}{KN(N-1)} P_1P_2$ becomes an equality, i.e.
        \begin{align*}
            \vec{w}_m-\vec{w}_n = c(\vec{h}_m^{(k)} - \vec{h}_n^{(k)}) \quad \text{for some } c'>0 \text{ not depending on $m,n,k$}
        \end{align*}
    \end{itemize}
\end{proof}

\begin{lemma}\label{lemma:equality_conditions}
Assume that the inequality (\ref{ineq:Cauchyschwarz}) shown in Lemma \ref{Lemma:CauchySchwarz} equalizes. Furthermore assume that there exist constants $c_{n,k}\in \R$ (depending on $n \in [N]$ and $k\in [K]$) and $c\in \R$ such that
\begin{align}
    \sprod{\vec{w}_m, \vec{h}_n^{(k)}} = c_{n,k} \quad \text{for every } m\in [N]\setminus \set{n}, \label{proof_equality_condition_1} \\
    \sum_{\substack{m=1 \\ m \neq n}}^N \sprod{\vec{w}_m-\vec{w}_n,\vec{h}_n^{(k)}} = c \; \text{(not depending on $n,k$)}, \label{proof_equality_condition_2}
\end{align}
for all $n \in [N]$ and $k \in [K]$. 
Then, the pair $(\vec{W},\vec{H})$ must form a neural collapse configuration. Conversely, if $(\vec{W},\vec{H})$ is a neural collapse configuration, then (\ref{ineq:Cauchyschwarz}) becomes an equality and the conditions (\ref{proof_equality_condition_1}, \ref{proof_equality_condition_2}) both hold true.

\end{lemma}

\begin{proof}
The converse implication is straightforward. We prove here the forward implication. By (\ref{eq1:4},\ref{proof_equality_condition_2}) we have for any $n\in [N]$ and $k\in [K]$ that
    \begin{align*}
        0 = \sum_{m=1}^N \sprod{\vec{w}_m,\vec{h}_n^{(k)}} =  \underbrace{\sum_{m=1}^N \sprod{\vec{w}_m-\vec{w}_n,\vec{h}_n^{(k)}}}_{=c} + N\sprod{\vec{w}_n,\vec{h}_n^{(k)}},
    \end{align*}
    so 
    \begin{align}\label{eq1:8}
        \sprod{\vec{w}_n,\vec{h}_n^{(k)}} = \frac{-c}{N}.
    \end{align}
    Combining this with (\ref{proof_equality_condition_2},\ref{proof_equality_condition_1}) gives
    \begin{align*}
        c = \sum_{\substack{m=1 \\ m \neq n}}^N \sprod{\vec{w}_m-\vec{w}_n, \vec{h}_n^{(k)}} = (N-1) c_{n,k} - (N-1) \frac{-c}{N},
    \end{align*}
    and hence
    \begin{align}\label{eq1:9}
        \sprod{\vec{w}_m,\vec{h}_n^{(k)}} = c_{n,k} = \frac{c}{N(N-1)}.
    \end{align}
    Combining (\ref{eq1:8},\ref{eq1:9}) with (\ref{eq1:5},\ref{eq1:7}) gives
    \begin{align}\label{eq1:10}
        \frac{-2c}{N-1}=\sprod{\vec{w}_n-\vec{w}_m,\vec{h}_n^{(k)} - \vec{h}_m^{(k)}} = c' \norm{\vec{h}_n^{(k)}-\vec{h}_m^{(k)}}^2 = c' \Big(\norm{\vec{h}_n^{(k)}}^2 + \norm{\vec{h}_n^{(k)}}^2\Big),
    \end{align}
     Combining (\ref{eq1:10}) with (\ref{eq1:6}) shows that for every $n \in [N]$ and $k \in [K]$, it holds
    \begin{align}\label{eq1:norm_b}
        \norm{\vec{h}_n^{(k)}}^2 = \frac{-c}{c'(N-1)}.
    \end{align}
    
    On the other hand, it follows also from (\ref{eq1:8},\ref{eq1:9}) that
    \begin{align}
        \norm{\vec{w}_n}^2 - \norm{\vec{w}_m}^2 = \sprod{\vec{w}_n-\vec{w}_m,\vec{w}_n+\vec{w}_m} = c' \sprod{\vec{h}_n^{(k)} - \vec{h}_m^{(k)}, \vec{w}_n+\vec{w}_m} = 0,
    \end{align}
    and hence the vectors $\vec{w}_n$, $n\in [N]$ have the same length, which can be computed via
    \begin{align*}
        N^2 \norm{\vec{w}_i}^2 &= N\sum_{n=1}^N \norm{\vec{w}_n}^2 = \sum_{n>m} \norm{\vec{w}_n-\vec{w}_m}^2 \\
        &= c' \sum_{n>m} \sprod{\vec{w}_n-\vec{w}_m,\vec{h}_n^{(k)}-\vec{h}_m^{(k)}} \\
        &= c' \cdot \frac{N(N-1)}{2} \cdot \frac{-2c}{N-1}\\
        &= -cc' N.
    \end{align*}
    Hence, for each $n\in [N]$, it holds 
    \begin{align}\label{eq1:norm_a}
        \norm{\vec{w}_n}^2 = \frac{-cc'}{N}.
    \end{align}

    Now let $\vec{h}^{(k)}: = \frac{1}{N} \sum_{m=1}^N \vec{h}_m^{(k)}$ for each $k \in [K]$. Observe that it holds
    \begin{align} \label{eq1:norm_aPb}
        \begin{split}
            \sprod{\vec{w}_n,\vec{h}_n^{(k)}-\vec{h}^{(k)}} &= \frac{N-1}{N}\sprod{\vec{w}_n,\vec{h}_n^{(k)}} - \frac{1}{N} \sum_{\substack{m=1\\m\neq n}}^N \sprod{\vec{w}_n,\vec{h}_m^{(k)}} \\
        &= -\frac{(N-1)c}{N^2} - \frac{c}{N^2} \\
        &= \frac{c}{N}.
        \end{split}
    \end{align}
    
    On the other hand, from (\ref{eq1:norm_b}) we have for each $n \in [N]$ and $k \in [K]$ that
    \begin{align} \label{eq1:norm_Pb}
        \begin{split}
            \norm{\vec{h}_n^{(k)}-\vec{h}^{(k)}}^2 &= \norm{\vec{h}_n^{(k)}}^2 - 2 \sprod{\vec{h}_n^{(k)}, \frac{1}{N}\sum_{m}\vec{h}_m^{(k)}} + \frac{1}{N^2} \norm{\sum_m \vec{h}_m}^2 \\
        &= \frac{N-1}{N} \norm{\vec{h}_n^{(k)}}^2 \\
        &= \frac{N-1}{N} \frac{-c}{c'(N-1)}\\
        &= \frac{-c}{Nc'}.    
        \end{split}
    \end{align}
    
    From (\ref{eq1:norm_a}-\ref{eq1:norm_Pb}) it follows that  
    \begin{align*}
        \sprod{\vec{w}_n,\vec{h}_n^{(k)} -\vec{h}^{(k)}} = \norm{\vec{w}_n} \norm{\vec{h}_n^{(k)} - \vec{h}^{(k)}},
    \end{align*}
    which implies that $\vec{w}_n$ is parallel to $\vec{h}_n^{(k)}-h^{(k)}$ for every $n \in [N]$ and $k \in [K]$. More precisely, by combining this finding with the above calculation of $\norm{\vec{w}_n}$ and $\norm{\vec{h}_n^{(k)}-h^{(k)}}$ in (\ref{eq1:norm_a}, \ref{eq1:norm_Pb}) we obtain
    \begin{align}\label{eq1:last}
        \vec{w}_n = c'\Big(\vec{h}_n^{(k)} -\vec{h}^{(k)}\Big).
    \end{align}
    Finally it is left to show that $\vec{h}_n^{(k)} = \vec{h}_n^{(\ell)}$ for any $k,\ell \in [K]$. For this observe that $\vec{h}_n^{(k)} - h^{(k)} = \vec{h}_n^{(\ell)}-h^{(\ell)} = \vec{w}_n$ implies 
    \begin{align*}
        \norm{\vec{h}_n^{(k)}}^2 &= \norm{\vec{h}_n^{(\ell)} + \vec{h}^{(k)}- \vec{h}^{(\ell)}}^2 \\
        &= \norm{\vec{h}_n^{(\ell)}}^2 + 2 \sprod{\vec{h}^{(k)} - \vec{h}^{(\ell)}, \vec{h}_n^{(k)}} + \norm{\vec{h}^{(k)} - \vec{h}^{(\ell)}}^2,
    \end{align*}
    and thus 
    \begin{align*}
        2 \sprod{\vec{h}^{(k)} - \vec{h}^{(\ell)}, \vec{h}_n^{(k)}} + \norm{\vec{h}^{(k)} - \vec{h}^{(\ell)}}^2 = 0.
    \end{align*}
    Similarly 
    \begin{align*}
        2 \sprod{\vec{h}^{(\ell)} - \vec{h}^{(k)}, \vec{h}_n^{(\ell)}} + \norm{\vec{h}^{(k)} - \vec{h}^{(\ell)}}^2 = 0.
    \end{align*}
    Combining the two equalities and taking the sum over $n$ we obtain
    \begin{align*}
        \norm{\vec{h}^{(k)} - \vec{h}^{(\ell)}}^2 = 0,
    \end{align*}
    which means that $\vec{h}^{(k)} = \vec{h}^{(\ell)}$ and therefore $\vec{h}_n^{(k)} = \vec{h}_n^{(\ell)}$. 
\end{proof}

\subsection{Proof of \Cref{theorem:standard_model}}\label{subsection:proof_standard_model}
\begin{proof}
\text{ }

\begin{enumerate}[label=Step \arabic*.]

    \item First we introduce a 
    lower bound on the (unregularized) loss. 
    Using Jensen's inequality for the convex function $t \mapsto e^t$ we obtain that for each $n \in [N]$ and $k \in [K]$ it holds
    \begin{align*}
        \sum_{\substack{m=1\\m\neq n}}^N e^{\sprod{\vec{w}_m-\vec{w}_n, \vec{h}_n^{(k)}}} \geq (N-1) e^{\frac{1}{N-1}\sum_{\substack{m=1\\m\neq n}}^N\sprod{\vec{w}_m-\vec{w}_n,\vec{h}_n^{(k)}}},
    \end{align*}
    with equality if and only if $\sprod{\vec{w}_m, \vec{h}_n} = c_n$ for every $m\neq n$, independently of $m$, for some constant $c_n$. Inserting this into the formulation of $L_\alpha$ in Section \ref{Sec: Reformulation_loss} we get 
    \begin{align*}
        L_\alpha(\vec{W},\vec{H}) \geq \frac{1}{NK}\sum_{n=1}^N \sum_{k=1}^K \Bigg[\log \Big( 1+(N-1)e^{\frac{1}{N-1} \sum_{\substack{m=1\\m\neq n}}^N\sprod{ \vec{w}_m-\vec{w}_n, \vec{h}_n^{(k)}}}\Big) \\
        \quad - \sum_{\substack{m=1\\m\neq n}}^N  \frac{\alpha}{N}  \sprod{\vec{w}_m-\vec{w}_n,\vec{h}_n^{(k)}}\Bigg].
    \end{align*}
    Observe that the function $t\mapsto \log \Big(1+(N-1)e^{\frac{t}{N-1}}\Big)$ is also convex, hence applying again Jensen's inequality we can lower bound the right-hand side in the estimate above, and obtain
    \begin{align}\label{Proof: Final Jensen inequality}
        \begin{split}
            L_\alpha(W,H) &\geq \log \Bigg( 1+ (N-1) e^{\frac{1}{KN(N-1)} \sum_{k=1}^K \sum_{n=1}^N \sum_{\substack{m=1\\m\neq n}}^N\sprod{\vec{w}_m-\vec{w}_n,\vec{h}_n^{(k)}}}\Bigg) \\
        &\quad-\frac{1}{NK}\sum_{k=1}^K \sum_{n=1}^N \sum_{\substack{m=1\\m\neq n}}^N  \frac{\alpha}{N}  \sprod{\vec{w}_m-\vec{w}_n,\vec{h}_n^{(k)}}.
        \end{split}
    \end{align}
    Equality in (\ref{Proof: Final Jensen inequality}) occurs if and only if the conditions (\ref{proof_equality_condition_1}, \ref{proof_equality_condition_2}) (see Lemma \ref{lemma:equality_conditions}) hold simultaneously.

    \item Recall that with the notation $P = P(\vec{W},\vec{H})$ from Lemma \ref{Lemma:CauchySchwarz}, the inequality (\ref{Proof: Final Jensen inequality}) becomes
    \begin{align}\label{eq1:1}
        \calL_\alpha(\vec{W},\vec{H}) \geq \log \Big( 1+ (N-1)e^P\Big) - \beta P + \lambda_W \norm{\vec{W}}^2 + \frac{\lambda_H}{K} \norm{\vec{H}}^2 =: \tilde{L}(\vec{W},\vec{H}),
    \end{align}
    with $\beta: = \frac{N-1}{N}\alpha >0$. Consider the function $g \colon \R \to \R$, 
    \begin{align*}
        g(t) := \log \Big( 1+(N-1)e^t\Big) - \beta t.
    \end{align*}
    Since $g$ is convex (as it differs from a convex function only by an additional linear function), it has a unique minimum specified as the root \footnote[1]{Note that here the root $t_0$ exists as long as $\beta >0$, for $\beta = 0$ we may, for convenience, define $t_0:= -\infty$ (this will correspond to Case (c) below). }
    \begin{align*}
        t_0:= \log \Big( \frac{1}{N-1 } \cdot \frac{\beta}{1-\beta} \Big) < 0
    \end{align*}
    of the derivative
    \begin{align*}
        g'(t) = \frac{(N-1)e^t}{1+(N-1)e^t}- \beta. 
    \end{align*}

    We now aim to find a constant lower bound on the right-hand side $\tilde{L}(\vec{W},\vec{H})$ of (\ref{eq1:1}). We consider the following three cases, corresponding to three different regions of the feasible set of $(\vec{W},\vec{H})$:

    \begin{enumerate}
        \item Case $t_0 > P(\vec{W},\vec{H})$: We will show that the minimizers of $\calL_\alpha$ cannot be in this region. Toward a contradiction, assume that there is a minimizer $(\vec{W}_0,\vec{H}_0)$ s.t. $P(\vec{W}_0,\vec{H}_0) < t_0$. We construct $(\vec{W}_1,\vec{H}_1)$ to be a NC configuration (according to Definition \ref{def_NC}) satisfying $\norm{\vec{W}_0}=\norm{\vec{W}_1}$ and $\norm{\vec{H}_0} = \norm{\vec{H}_1}$. Then we have
\begin{align*}
    P(\vec{W}_1,\vec{H}_1) &= -\frac{1}{\sqrt{K(N-1)}}\norm{\vec{W}_1}\norm{\vec{H}_1} \\
    &= -\frac{1}{\sqrt{K(N-1)}}\norm{\vec{W}_0}\norm{\vec{H}_0} \\
    &\leq P(\vec{W}_0,\vec{H}_0)\\
    &<t_0 <0.
\end{align*}
By rescaling $\vec{W}_1,\vec{H}_1$ (with a constant smaller than 1) we obtain a pair $(\vec{W},\vec{H})$ with $P(\vec{W},\vec{H}) = t_0$ and $\norm{\vec{W}}< \norm{\vec{W}_0}$, $\norm{\vec{H}} < \norm{\vec{H}_0}$. Thus it holds
\begin{align*}
    \calL_\alpha(\vec{W}_0,\vec{H}_0) &\geq \log \Big( 1+(N-1)e^{P(\vec{W}_0,\vec{H}_0)} \Big)-\beta P(\vec{W}_0,\vec{H}_0) \\ 
    & \quad+ \lambda_W \norm{\vec{W}_0}^2 + \lambda_H \norm{\vec{H}_0}^2\\
    &> \log \Big( 1+(N-1)e^{t_0} \Big)-\beta t_0 + \lambda_W \norm{\vec{W}}^2 + \frac{\lambda_H}{K} \norm{\vec{H}}^2\\
    &= \calL_\alpha(\vec{W},\vec{H}),
\end{align*}
which means that $(\vec{W}_0,\vec{H}_0)$ cannot be a minimizer of $\calL_\alpha$. Note that the last equality holds because the inequality (\ref{Proof: Final Jensen inequality}) equalizes when $(\vec{W},\vec{H})$ is a NC configuration (see Lemma \ref{lemma:equality_conditions}).

        \item Case $P(\vec{W},\vec{H}) \geq t_0 \geq -\frac{1}{\sqrt{K(N-1)}}\norm{\vec{W}}\norm{\vec{H}}$: We will show that at the minimizers in this region, $P$ must be $t_0$. Assume that $(\vec{W}_0,\vec{H}_0)$ is a minimizer of $\tilde{L}$ in this region with $P(\vec{W}_0,\vec{H}_0) \neq t_0$. Then we consider all pairs $(\vec{W},\vec{H})$ with $\norm{\vec{W}} \leq \norm{\vec{W}_0}$ and $\norm{\vec{H}}\leq \norm{\vec{H}_0}$. By continuity we have that $P(\vec{W},\vec{H})$ can take all values in the interval
\begin{align*}
   \Big[ -\frac{1}{\sqrt{K(N-1)}} \norm{\vec{W}_0}\norm{\vec{H}_0}, \frac{1}{\sqrt{K(N-1)}} \norm{\vec{W}_0}\norm{\vec{H}_0} \Big],
\end{align*}
which also includes $t_0$. It follows that $\tilde{L}(\vec{W},\vec{H}) < \tilde{L}(\vec{W}_0,\vec{H}_0)$, so $(\vec{W}_0,\vec{H}_0)$ cannot be a minimizer of $\tilde{L}$, meaning that a minimizer $(\vec{W},\vec{H})$ of $\tilde{L}$ must satisfy $P(\vec{W},\vec{H}) = t_0$. The minimization of $\tilde{L}$ then reduces to 
\begin{align*}
    \min_{W,H} \quad \lambda_W \norm{W}^2 + \lambda_H \norm{H}^2 \quad \text{s.t. } \quad -\frac{1}{\sqrt{K(N-1)}} \norm{W}\norm{H} = t_0.
\end{align*}

        Observe that 
        \begin{align*}
            \lambda_W \norm{\vec{W}}^2 + \frac{\lambda_H}{K} \norm{\vec{H}}^2 &\geq 2\sqrt{\frac{\lambda_W \lambda_H}{K} } \norm{\vec{W}} \norm{\vec{H}} \\
            &\geq -2t_0\sqrt{(N-1)\lambda_W \lambda_H}.
        \end{align*}
        Therefore we have $\tilde{L}(W,H) \geq g(t_0) -2 t_0\sqrt{(N-1)\lambda_W \lambda_H}$ and this equalizes if and only if the following conditions hold:
        \begin{itemize}
            \item $P(\vec{W},\vec{H}) = t_0$
            \item $\lambda_W \norm{\vec{W}}^2 = \lambda_H \norm{\vec{H}}^2$ and $\norm{\vec{W}}\norm{\vec{H}} = -\sqrt{K(N-1)}t_0$.
        \end{itemize}

        \item Case $P(\vec{W},\vec{H}) \geq -\frac{1}{\sqrt{K(N-1)}} \norm{\vec{W}} \norm{\vec{H}} \geq t_0$: 
        
        In this region, it holds $g\Big(P(\vec{W},\vec{H})\Big) \geq g\Big(-\frac{1}{\sqrt{K(N-1)}} \norm{\vec{W}} \norm{\vec{H}}\Big)$, so
        \begin{align*}
            \tilde{L}(\vec{W},\vec{H}) \geq f(\norm{\vec{W}}, \norm{\vec{H}}),
        \end{align*}
        with $f \colon \R^2 \to \R$,
        \begin{align*}
            f(w,h) := \log \Big( 1+ (N-1)e^{-Cwh}\Big) + \beta Cwh +\lambda_W w^2 + \frac{\lambda_H}{K} h^2
        \end{align*}
        where we set $C:= \frac{1}{\sqrt{K(N-1)}}$ to shorten notation. Observe that even though $w$ and $h$, as representatives for $\norm{\vec{W}}$ and $\norm{\vec{H}}$ respectively, must be positive, we can consider them as real number (without positivity). This can be explained as follows. On the one hand, we are interested in the global minimum of $f$, at which $w$ and $h$ should have the same sign. On the other hand, since $f(w,h)=f(-w,-h)$, if $(w,h)$ is a minimum point then certainly $(-w,-h)$ is a minimum point of $f$. 
        
        This observation allows us to set the derivatives of $f$ to be $0$ at the minimum, i.e. 
        \begin{align*}
            0 =\nabla_w f(w,h) = -\frac{(N-1)e^{-Cwh}}{1+(N-1)e^{-Cwh}} Cb +2\beta Ch + 2\lambda_W w,\\
            0 = \nabla_h f(w,h) = -\frac{(N-1)e^{-Cwh}}{1+(N-1)e^{-Cwh}} Ca +2\beta Cw+ 2\frac{\lambda_H}{K} h.
        \end{align*}
        Multiplying the first equality with $w$ and the second with $h$, we obtain in particular that $\lambda_W w^2 = \frac{\lambda_H}{K} h^2$, and hence $h = \sqrt{\frac{K\lambda_W}{\lambda_H}} w$. Inserting this into the first inequality while denoting $C': = C \sqrt{\frac{K\lambda_W}{\lambda_H}}$ yields
        \begin{align*}
            -\frac{(N-1)e^{-C' a^2}}{1+(N-1)e^{-C' w^2}}C' w + 2\beta C' w + 2 \lambda_W w= 0.
        \end{align*}
       Excluding the trivial solution $(w,h) = (0,0)$, so that we can multiply both sides with $1/w$, we get 
       \begin{align}\label{eq1:2}
           w^2= \frac{1}{C} \sqrt{\frac{\lambda_H}{K\lambda_W}} \log \Bigg( (N-1) \frac{1-\beta - 2\sqrt{(N-1)\lambda_W \lambda_H}}{\beta + 2\sqrt{(N-1)\lambda_W \lambda_H}}\Bigg)
       \end{align}
        and 
        \begin{align}\label{eq1:3}
           h^2= \frac{1}{C} \sqrt{\frac{K\lambda_W}{\lambda_H}} \log \Bigg( (N-1) \frac{1-\beta - 2\sqrt{(N-1)\lambda_W \lambda_H}}{\beta + 2\sqrt{(N-1)\lambda_W \lambda_H}
           }\Bigg)
       \end{align}
        
    Finally, it is easy to check that 
    \begin{align*}
        -Cwh = \log \Big( \frac{1}{N-1} \frac{\beta+ 2\sqrt{(N-1)\lambda_W \lambda_H}}{1-\beta-2\sqrt{(N-1)\lambda_W \lambda_H}}\Big) > \log\Big( \frac{1}{N-1} \cdot \frac{\beta}{1-\beta}\Big) = t_0,
    \end{align*}
    i.e. the solution found above belongs indeed to the current region of the feasible set. In summary, we have shown in this case that $\tilde{L}(\vec{W},\vec{H}) \geq f(w_0,h_0)$ with $(w_0,h_0)$ specified as in (\ref{eq1:2}, \ref{eq1:3}), and this becomes equality if and only if $P(\vec{W},\vec{H}) = -\frac{1}{\sqrt{K(N-1)}} \norm{\vec{W}}\norm{\vec{H}}$ and $\norm{\vec{W}} =w_0$, $\norm{\vec{H}} = h_0$.
    
    \end{enumerate}

    \item We now come back to the actual loss $\calL_\alpha$. In both cases (b) and (c) discussed above, we have shown that $\calL_\alpha(\vec{W},\vec{H}) \geq \tilde{L}(\vec{W},\vec{H}) \geq \text{const}$ and this can equalize when the conditions in Lemma \ref{lemma:equality_conditions} are satisfied. We deduce that $\calL_\alpha$ achieves its minimum at either case (b) or (c), while both lead to a NC configuration by Lemma \ref{lemma:equality_conditions}.

\end{enumerate}

\end{proof}

\section{Theoretical supports for Theorem \ref{theorem:MD_model}}

In this appendix we prove our theoretical result on the MD model, namely Theorem \ref{theorem:MD_model}. 

\subsection{Preparation for the proof}\label{Section:MD_prepare_proof}
The problem from Definition \ref{def:MD_model} is
\begin{align*}
    \min_{U\geq 0,r\geq 0} \quad \mathcal{R}_{\lambda,\eta,\alpha} (\vec{U},r) := F_{\lambda,\alpha} (\vec{W},\vec{H},r) + \eta G_{\lambda,\alpha}(\vec{W},\vec{U},r)
\end{align*}
under the constraints
\begin{align*}
    \eta \norm{\vec{h}_1-\vec{u}_2} \leq \frac{C_{MD}r}{\norm{\vec{h}_1-\vec{h}_2}},\\
    \eta \norm{\vec{h}_2-\vec{u}_1} \leq \frac{C_{MD}r}{\norm{\vec{h}_1-\vec{h}_2}}.\\
\end{align*}

Observe that $F_{\lambda,\alpha}$ does not depend on $\vec{U}$. Hence, for each $r\geq 0$ we can first solve the problem 
\begin{align*}
    \min_{\vec{U} \geq 0} \quad G_{\lambda, \alpha} (\vec{W},\vec{U},r) 
\end{align*}
under the same constraints to obtain the optimal configuration of $\vec{U} = \vec{U}(r)$, and then solve 
\begin{align*}
    \min_{r\geq 0} \quad \mathcal{R}_{\lambda, \eta,\alpha} \Big(\vec{U}(r),r\Big).
\end{align*}
The problem of optimizing $G_{\lambda, \alpha} (\vec{W},\vec{U},r) $ over $\vec{U}$ can be separated into two subproblems over $\vec{u}_1$ and $\vec{u}_2$, which are independent and symmetric. We hence consider only the problem over $\vec{u}_1$, namely
\begin{align}\label{problem:u1}\tag{$\calP_{\vec{u}_1}$}
\begin{split}
    &\min_{\vec{u}_1 \in \R_+^M} \quad \log \Big(1 + e^{\sprod{\vec{w}_2-\vec{w}_1,\vec{u}_1}} \Big) - \frac{\alpha}{2} \sprod{\vec{w}_2-\vec{w}_1,\vec{u}_1}+ \lambda \norm{\vec{u}_1}^2\\
    &\text{ s.t. } \quad \eta\norm{\vec{h}_2-\vec{u}_1} \leq \frac{C_{MD} r}{\norm{\vec{h}_1-\vec{h}_2}} .
\end{split}
\end{align}

\begin{remark}\label{remark:h1_opt}
Without its constraint, the minimization of $G_{\lambda,\alpha}(\vec{W},\vec{U},r)$ over $\vec{U}$ becomes a reduction of the problem
\begin{align*}
    \min_{\vec{W},\vec{H}} \quad \ell_\alpha(\vec{W}, \vec{h}_1, y_1^{(\alpha)})+\ell_\alpha(\vec{W}, \vec{h}_2, y_2^{(\alpha)}) + \lambda_W \norm{\vec{W}}^2 + \lambda_H \norm{\vec{H}}^2
\end{align*}
where $\vec{W}$ is restricted to be in the optimal NC configuration (see \Cref{def_NC}). Hence the problem (\ref{problem:u1}) without its constraint has the minimizer at $\vec{u}_1=\vec{h}_1$. Furthermore the problem (\ref{problem:u1}) itself also has its minimizer at $\vec{h}_1$ if $\vec{h}_1$ is feasible under the side constraint. Namely, when
\begin{align*}
    \eta\norm{\vec{h}_2-\vec{h}_1} \leq \frac{C_{MD} r}{\norm{\vec{h}_1-\vec{h}_2}},
\end{align*}
or equivalently when
\begin{align*}
    r \geq \frac{\eta \norm{\vec{h}_1-\vec{h}_2}^2 }{C_{MD}} =: r_{\max}.
\end{align*}
Thus we only need to study the problem (\ref{problem:u1}) in case $r < r_{\max}$. 
\end{remark}

The rest of the proof can be summarized as follows: First, in Subsection \ref{section:est_solution_P_u1} we show that the solution $u_1(r)$ to (\ref{problem:u1}) must be on a small subset of the feasible set (see Lemma \ref{Lemma:P_u1} and \ref{Lemma:P_u2}), which allows us to prove the (almost) linear dependence of $u_1(r) - u_1(r_{\max})$ on the distance $r_{\max}-r$ (see Lemma \ref{lemma:problem_u3}). Next, we study the behavior of $r\mapsto G_{\lambda,\alpha}\Big( \vec{W}, \vec{U}(r), r\Big)$ locally around $r_{\max}$ and the behavior of $r \mapsto F_{\lambda,\alpha} (\vec{W}, \vec{H},r)$ around $0$. This lets us show that the decay of the former function near $r_{\max}$ dominates the increasing of the latter one, and hence the optimal dilation $r_*$ must be close to $r_{\max}$. The details of this argument are introduced in Subsection \ref{section:optimal_dilation}. Finally in Subsection \ref{section:finalize_proof} we apply this to each value $\alpha \in \{0,\alpha_0\}$ and get the desired statement of Theorem \ref{theorem:MD_model}.

\subsection{Estimating the solution of \cref{problem:u1}} \label{section:est_solution_P_u1}

For convenience, in this section we introduce several notations. Let 
\begin{align*}
    \calS:= \operatorname{span} \set{\vec{w}_2-\vec{w}_1,\vec{h}_1} = \operatorname{span} \set{\vec{h}_2-\vec{h}_1,\vec{h}_1} = \operatorname{span} \set{\vec{h}_2,\vec{h}_1}
\end{align*}
be the two-dimensional subspace spanned by $\vec{w}_2-\vec{w}_1$ and $\vec{h}_1$. Furthermore, let $\mathcal{B}$ be the ball of radius $\frac{C_{MD}r}{\eta \norm{\vec{h}_1-\vec{h}_2}}$ around $\vec{h}_2$. Let $\mathcal{C} = \partial \mathcal{B} \cap \calS$ be the circle that is the intersection of the ball $\mathcal{B}$ and the subspace $\calS$. We will show that the minimizer of \cref{problem:u1} must lie on the circle $\mathcal{C}$. Note that the feasiblity of a vector $x \in \R^M$ for the problem \cref{problem:u1} can be expressed as $x \in \R_+^M \cap \mathcal{B}$.

\begin{lemma}\label{Lemma:P_u1}
Let $r< r_{\max}$, then the minimizer of (\ref{problem:u1}) lies on the circle $\mathcal{C}$, i.e. it lies on the subspace $\calS$ and the inequality constraint in (\ref{problem:u1}) must equalize at the minimizer. 
\end{lemma}

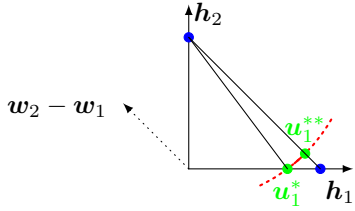
\begin{figure}[h]
        \centering
        \begin{tikzpicture}[scale=1.75,cap=round,>=latex]
            \draw[->] (0cm,0cm) -- (1.25cm,0cm) ;
            \draw[->] (0cm,0cm) -- (0cm,1.25cm) ;
            \draw[dotted,->] (0cm,0cm) -- (-0.5cm,0.5cm) ;
            \draw (-1cm,0.3cm) node[above=1pt] {$\vec{w}_2-\vec{w}_1$};
            
    
            \draw[dotted,red,thick] ([shift =(-30: 1.25)]0,1) arc (-30:-65:1.25cm);
            \draw[red,thick] ([shift =(-45: 1.25)]0,1) arc (-45:-50:1.25cm);
            
            
            \filldraw[blue] (1cm,0cm) circle(1pt);
            \filldraw[blue] (0cm,1cm) circle(1pt);
            \filldraw[green] (0.75cm,0cm) circle(1pt) node[below=1pt] {$\vec{u}_1^{*}$};
            \filldraw[green] ([shift =(-45: 1.25)]0,1) circle(1pt) node[above=1pt] {$\vec{u}_1^{**}$};
            
            \draw (0.15cm,1cm) node[above=1pt] {$\vec{h}_2$};
            \draw (1.15cm,0cm) node[below=1pt] {$\vec{h}_1$};

            \draw[-] (0cm,1cm) -- (1cm,0cm);
            \draw[-] (0cm,1cm) -- (0.75cm,0cm);
            
        \end{tikzpicture}
        
        \caption{Illustration for Lemma \ref{Lemma:P_u1} and \ref{Lemma:P_u2}. The feasible set of (\ref{problem:u1}) is the intersection of the positive quadrant and the disc with boundary given by the red circle $\mathcal{C}$. We consider the case $\mathcal{C}$ has an intersection $\vec{u}_1^*$ with the segment $(0,\vec{h}_1)$ and $\vec{u}_1^{**}$ with the segment $(\vec{h}_1,\vec{h}_2)$. The minimizer of (\ref{problem:u1}) must lie on the red arc between $\vec{u}_1^*$ and $\vec{u}_1^{**}$.}
        \label{Fig:u1}
    \end{figure}

\begin{proof}
\text{ }

\begin{enumerate}
    \item Let $x \in \R_+^M \cap \mathcal{B}$ be a feasible solution. According to $\R^M = \calS \oplus \calS^\bot$ we can decompose $x$ into 
    \begin{align*}
        x = \vec{x}_\calS + x- \vec{x}_\calS,
    \end{align*}
    where $\vec{x}_\calS$ is the orthogonal projection of $x$ on the subspace $\calS$ and $x-\vec{x}_\calS$ is orthogonal to $\calS$. We will show that $\vec{x}_\calS$ is a better candidate for \cref{problem:u1} than $x$, which means that $\vec{x}_\calS$ is also feasible and leads to smaller objective value. The second point is straightforward, because one can observe that 
    \begin{align*}
        \sprod{x,w_2-w_1} = \sprod{\vec{x}_\calS, w_2-w_1},
    \end{align*}
    and 
    \begin{align*}
        \norm{x}^2 = \norm{\vec{x}_\calS}^2 + \norm{x-\vec{x}_\calS}^2 \geq \norm{\vec{x}_\calS}^2.
    \end{align*}
    Thus it is left to show the feasibility of $\vec{x}_\calS$. For this, observe that $\vec{h}_1$ and $\vec{h}_2$ form an (entrywise) nonnegative orthogonal basis of $\calS$ (remark: not necessarily an orthonormal basis, because $\vec{h}_1$ and $\vec{h}_2$ are not necessarily normalized). Thus $\vec{x}_\calS$ can be written as
    \begin{align*}
        \vec{x}_\calS = a_1 \vec{h}_1 + a_2 \vec{h}_2
    \end{align*}
    in which the coefficients $a_1,a_2 \in \R$ satisfy
    \begin{align*}
        a_i = \sprod{\vec{x}_\calS,\frac{\vec{h}_i}{\norm{\vec{h}_i}^2}} = \sprod{x,\frac{\vec{h}_i}{\norm{\vec{h}_i}^2}} \geq 0,
    \end{align*}
    where the last inequality holds because both vectors $\vec{x}$ and $\vec{h}_i$ are nonnegative. 
    
    Finally we need to show that $\vec{x}_\calS \in \mathcal{B}$. For this consider again the decomposition 
    \begin{align*}
        \vec{x} = \vec{x}_\calS + \vec{x}-\vec{x}_\calS = a_1\vec{h}_1+a_2\vec{h}_2 + (\vec{x}-\vec{x}_S).
    \end{align*}
    We obtain
    \begin{align*}
        \norm{\vec{h}_2-\vec{x}}^2 &= \norm{a_1\vec{h}_1 + (a_2-1)\vec{h}_2 + \vec{x}-\vec{x}_\calS}^2 \\
        &= a_1^2\norm{\vec{h}_1}^2 + (a_2-1)^2\norm{\vec{h}_2}^2 + \norm{\vec{x}-\vec{x}_\calS}^2\\
        &\geq a_1^2\norm{\vec{h}_1}^2 + (a_2-1)^2\norm{\vec{h}_2}^2\\
        &= \norm{a_1\vec{h}_1 + (a_2-1)\vec{h}_2}^2 \\
        &= \norm{\vec{h}_2-\vec{x}_\calS}^2,
    \end{align*}
    and hence the claim $\vec{x}_\calS \in \mathcal{B}$ follows from the feasibility of $\vec{x}$.

    \item Now we reduce to the subspace $\calS$. Since $\vec{h}_1/\norm{\vec{h}_1}$ and $\vec{h}_2/\norm{\vec{h}_1}$ form a nonnegative orthonormal basis for this subspace, it is equivalent to consider the space $\R^2$ of the coefficients. Note that the positivity of a vector $\vec{x} \in \calS \subseteq \R^M$ is also equivalent to the positivity of its coefficients in $\R^2$. Thus for convenience we may assume without loss of generality that $\vec{h}_1$ and $\vec{h}_2$ are orthogonal vectors in $\R^2$, more precisely $\vec{h}_1$ lies on the $x$-axis and $\vec{h}_2$ lies on the $y$-axis as in Figure \ref{Fig:u1}.
    
    We denote by $\mathcal{F} \subset \R^2_+$ the set of all points in $\calS$ that is feasible to (\ref{problem:u1}), i.e. inside the ball $\mathcal{B}$ around $\vec{h}_2$. Note that on the two-dimensional subspace $\calS$ the ball $\mathcal{B}$ reduces to a disc, whose boundary is the circle $\mathcal{C}$ as defined above. We will show that the solution to \cref{problem:u1} must lie on the boundary of $\calF$ (w.r.t. the topology in $\calS \cong \R^2$), i.e either on the axes (due to positivity constraints) or on the circle $\mathcal{C}$. 
    
    Indeed, let $\vec{u}_1$ be an arbitrary feasible point in the interior of $\mathcal{F}$, we prove that $\vec{u}_1$ is not the minimum of the problem (\ref{problem:u1}). Since $\vec{u}_1$ is an interior point, there exists a disc $\mathcal{B}'$ around $\vec{u}_1$ which lies completely inside $\mathcal{F}$. Let $\mathcal{A}$ be the intersection of the disc $\mathcal{B'}$ and the circle of radius $\norm{\vec{u}_1}$ around the origin. Observe that moving $\vec{u}_1$ along the arc $\mathcal{A}$ keeps its norm unchanged, but can both increase and decrease the value of $\sprod{\vec{w}_2-\vec{w}_1,\vec{u}_1}$. Hence the objective in (\ref{problem:u1}) cannot reach its minimum at $\vec{u}_1$, unless $\sprod{\vec{w}_2-\vec{w}_1,\vec{u}_1}$ is equal to the minimizer $t_0$ of the function
    \begin{align*}
        t \mapsto \log (1+e^t) - \frac{\alpha}{2}t.
    \end{align*}
    However, in case $\sprod{\vec{w}_2-\vec{w}_1,\vec{u}_1} = t_0$, the point $\vec{u}_1$ lies on a line that is orthogonal to $\vec{w}_2-\vec{w}_1$, and one can choose another point $\vec{u}_1'$ on the intersection of this line and the disc $\mathcal{B}'$ such that $\norm{\vec{u}_1'} < \norm{\vec{u}_1}$. In particular, $\vec{u}_1'$ is a better feasible candidate in comparison to $\vec{u}_1$. 
    
    \item We have shown above that the interior point $\vec{u}_1$ cannot be the solution of \cref{problem:u1}. Excluding all interior points, we now consider the boundary set $\partial \calF$, which consists of points on the circle $\mathcal{C}$ that lie in the positive quadrant (denoted by $\partial \calF_1$), points on the $x$-axis between $0$ and the intersection point $\vec{u}_1^*$ of $\mathcal{C}$ with the $x$-axis (denoted by $\partial \calF_2$), as well as points on the $y$-axis between $0$ and the intersection point of $\mathcal{C}$ with the $y$-axis (denoted by $\partial \calF_3$). Note that the circle $\mathcal{C}$ may have no intersection with the $x$-axis, in that case we simply consider $\partial\calF_2$   as the empty set. 
    
    We will show that the solution must be a point on $\partial \calF_1$, in which we find the possible optimal positions of $\vec{u}_1$ on each of the other boundary subsets, i.e. $\partial \calF_2$ and $\partial \calF_3$. 
    \begin{enumerate}
    \item On $\partial \calF_3$: Observe that moving a point $\vec{u}_1$ along $\partial \calF_3$ in the direction toward the origin will decrease both the scalar product $\sprod{\vec{w}_2-\vec{w}_1,\vec{u}_1}>0$ (as the angle is kept unchanged while the length of $\vec{u}_1$ is decreased) and the regularization term $\norm{\vec{u}_1}^2$. Since the function $t \mapsto \log(1+e^t) -\frac{\alpha}{2}t$ is monotonically increasing on $[0,\infty)$, moving $\vec{u}_1$ in this direction decreases the objective in (\ref{problem:u1}). Therefore, the best candidate on $\partial \calF_3$ is the lowest possible point on $\partial \calF_3$, i.e. $0$ in case $\mathcal{C}$ has intersection point with the $x$-axis, or is the lower intersection point of $\mathcal{C}$ with the $y$-axis otherwise. 
    
    \item On $\partial \calF_2$ (in case it is not empty): Here, the objective becomes $f(\norm{\vec{u}_1})$ where the function $f$ is defined by
    \begin{align*}
        f(t) = \log (1+e^{c_1t}) - \frac{\alpha}{2} c_1 t + \lambda t^2,
    \end{align*}
    with $c_1:= \frac{\sprod{\vec{w}_2-\vec{w}_1,\vec{h}_1}}{\norm{\vec{h}_1}}$. Observe that $f$ is convex (this can be seen by directly computing the 2nd derivative of $f$) and achieves its minimum at $t = \norm{\vec{h}_1}$ (because $\vec{h}_1$ is the minimum of (\ref{problem:u1}) without the side constraint, see Remark \ref{remark:h1_opt}). Hence on the interval $[0, \norm{\vec{h}_1}]$ it is monotonically decreasing. It follows that $\vec{u}_1^*$ is the best candidate on $\partial \calF_2$. 
    \end{enumerate}

    In summary we have shown that the optimal position of $\vec{u}_1$ must be on $\partial \calF_1$, $\partial \calF_2$ or $\partial \calF_3$. On the other hand, all candidates on $\partial \calF_3$ are worse than a point in $\partial \calF_2$ and all candidates on $\partial \calF_2$ are worse than a point in $\partial \calF_1$ (in case $\partial \calF_2 = \emptyset$ we have that all candidates on $\partial \calF_3$ are worse than a point in $\partial \calF_1$). Therefore the minimizer must be a point on $\partial \calF_1$, in particular on the circle $\mathcal{C}$. 
\end{enumerate}

\end{proof}

Having said that the optimal position of $\vec{u}_1$ with respect to (\ref{problem:u1}) must be on the circle $\mathcal{C}$, we are now interested in the case where $r$ is close to $r_{\max}$, in which the circle $\mathcal{C}$ has intersection with the segment $(0,\vec{h}_1)$ (see Figure \ref{Fig:u1}). In this case we can even restrict the possible optimal positions to a smaller subset of the circle.

\begin{lemma}\label{Lemma:P_u2}
Suppose that $r_{\max}\geq r \geq r_{\max}/\sqrt{2}$, so that the circle $\mathcal{C}$ has intersection $\vec{u}_1^*$ with the line segment $(0,\vec{h}_1)$ and intersection $\vec{u}_1^{**}$ with the line segment $(\vec{h}_2,\vec{h}_1)$. Then, the minimizer of (\ref{problem:u1}) lies on the arc between $\vec{u}_1^*$ and $\vec{u}_1^{**}$. 
\end{lemma}

\begin{proof}
First we rewrite the objective of (\ref{problem:u1}) as
\begin{align*}
    f\big(\sprod{\vec{w}_2-\vec{w}_1,\vec{u}_1} \big) + \lambda \norm{\vec{u}_1}^2.
\end{align*}
where $f: \R \to \R$ is the function defined by  
\begin{align*}
    f(t) =  \log \Big( 1+e^t\Big) - \frac{\alpha}{2}t.
\end{align*}

Next, we parameterize the circle $\mathcal{C}$ by the polar coordinate. Let $R:= \frac{C_{MD}r}{\eta \norm{\vec{h}_1-\vec{h}_2}}$ and $\theta$ be the angle between $(\vec{h}_2,\vec{u}_1)$ and $(\vec{h}_2,\vec{h}_1)$. Then, since $\vec{w}_2-\vec{w}_1$ is proportional to $\vec{h}_2-\vec{h}_1$ we have
\begin{align*}
    \sprod{\vec{w}_2-\vec{w}_1,\vec{u}_1} &= \sprod{\vec{w}_2-\vec{w}_1,\vec{h}_2}+ \sprod{\vec{w}_2-\vec{w}_1, \vec{u}_1-\vec{h}_2}\\
    &= \sprod{\vec{w}_2-\vec{w}_1,\vec{h}_2}- R\norm{\vec{w}_2-\vec{w}_1}\cos{\theta}.
\end{align*}
Note that $\vec{u}_1$ can be on both sides of the line $(\vec{h}_2,\vec{h}_1)$, but for the calculation of $\sprod{\vec{w}_2-\vec{w}_1,\vec{u}_1}$ it is not necessary to distinguish between these two cases. In general, when $\theta$ increases, $\cos \theta$ decreases (we can exclude the case $\theta > \pi/2$ because in this case $\sprod{\vec{w}_2-\vec{w}_1,\vec{u}_1}$ becomes positive and the norm of $\vec{u}_1$ is also large, so the objective becomes large and $\vec{u}_1$ cannot be the minimizer), and thus $\sprod{\vec{w}_2-\vec{w}_1,\vec{u}_1}$ increases.

Now we claim that the optimal position of $\vec{u}_1$ must be on the arc between $\vec{u}_1^*$ and its reflection $\vec{u}_1'$ about the line $(\vec{h}_1,\vec{h}_2)$. To show this we consider a point $\vec{u}_1$ that lies on the other part of the circle $\mathcal{C}$. By the above observation on the monotonicity of $\sprod{\vec{w}_2-\vec{w}_1,\vec{u}_1}$ with respect to $\theta$ we see that 
\begin{align*}
    \sprod{\vec{w}_2-\vec{w}_1,\vec{u}_1} &> \sprod{\vec{w}_2-\vec{w}_1,\vec{u}_1'} \\
    &= \sprod{\vec{w}_2-\vec{w}_1,\vec{u}_1^*} \\
    &\geq \sprod{\vec{w}_2-\vec{w}_1,\vec{h}_1}.
\end{align*}
Recall from the proof of Theorem \ref{theorem:standard_model} that $\sprod{\vec{w}_2-\vec{w}_1,\vec{h}_1}$ is not smaller than the minimizer $t_0$ of $f$, and due to convexity $f$ is monotone increasing on $[t_0,\infty)$. Therefore we obtain
\begin{align*}
    f\Big( \sprod{\vec{w}_2-\vec{w}_1,\vec{u}_1}\Big) > f\Big( \sprod{\vec{w}_2-\vec{w}_1,\vec{u}_1'}\Big)
\end{align*}
On the other hand, by the law of cosines applied to the triangle $(0,\vec{h}_2,\vec{u}_1)$ we obtain
\begin{align*}
    \norm{\vec{u}_1}^2 = \norm{\vec{h}_2}^2 + R^2 - 2 R\norm{\vec{h}_2} \cos \Big( \frac{\pi}{4}+ \theta \Big),
\end{align*}
which is increasing in $\theta$ (again we exclude the case $\theta > \pi/2$ as discussed above). This shows that $\norm{\vec{u}_1} > \norm{\vec{u}_1'}$. Combining the above two inequalities we see that $\vec{u}_1'$ is a better candidate than $\vec{u}_1$. 

Finally, the desired statement follows from the observation that we can exclude all points on the arc between $\vec{u}_1^{**}$ and $\vec{u}_1'$, because each point on this arc can be reflected about the line $(\vec{h}_1,\vec{h}_2)$ to a point with the same value of $\sprod{\vec{w}_2-\vec{w}_1,\vec{u}_1}$, but with smaller norm and this gives a better value of the objective.

\end{proof}

Note that similar to the optimal position of $\vec{U}$, the points $\vec{u}_1^*$, $\vec{u}_1^{**}$ from Lemma \ref{Lemma:P_u2} also depend on $r$. Hence to be clear, we may write $\vec{u}_1 = \vec{u}_1(r)$, $\vec{u}_1^* = \vec{u}_1^*(r)$ and $\vec{u}_1^{**}= \vec{u}_1^{**}(r)$ for $r\in [r_{\max}/\sqrt{2}, r_{\max}]$. Observe that 
\begin{align*}
    \vec{u}_1(r_{\max}) = \vec{u}_1^{*}(r_{\max}) = \vec{u}_1^{**}(r_{\max}) = \vec{h}_1.
\end{align*}
The following lemma shows that for $r$ close to $r_{\max}$, the distance between $\vec{u}_1(r)$ and $\vec{h}_1$ behaves almost linearly with respect to the distance between $r$ and $r_{\max}$.

\begin{lemma}\label{lemma:problem_u3}
Let $\vec{u}_1(r)$ be the minimizer of (\ref{problem:u1}) with input $r \in [r_{\max}/\sqrt{2}, r_{\max}]$. Then, there exists constants $c,C>0$ (depending on $\norm{\vec{h}_1} =\norm{\vec{h}_2} $ and $C_{MD}$, but not on other parameters such as $r$, $\eta$, etc) such that 
\begin{align*}
    \norm{\vec{u}_1(r) -\vec{h}_1} \in  \Big(c\frac{r_{\max}-r}{\eta}, C\frac{r_{\max}-r}{\eta} \Big).
\end{align*}

\end{lemma}

\begin{proof}
From Lemma \ref{Lemma:P_u2} we know that $\vec{u}_1(r)$ lies on the arc between $\vec{u}_1^*(r)$ and $\vec{u}_1^{**}(r)$, hence its distance to $\vec{h}_1$ is lower bounded by the distance from $\vec{u}_1^{**}(r)$ to $\vec{h}_1$ and is upper bounded by the distance from $\vec{u}_1^*(r)$ to $\vec{h}_1$. 
Hence we have
\begin{align*}
    \norm{\vec{u}_1(r) -\vec{h}_1} &\leq \norm{\vec{u}_1^*(r)-\vec{u}_1^*(r_{\max})} \\
    &= \norm{\vec{u}_1^*(r_{\max})} - \norm{\vec{u}_1^*(r)}\\
    &= \sqrt{\norm{\vec{h}_2-\vec{u}_1^*(r_{\max})}^2 - \norm{\vec{h}_2}^2} - \sqrt{\norm{\vec{h}_2-\vec{u}_1^*(r)}^2 - \norm{\vec{h}_2}^2}\\
    &= \sqrt{\frac{C_{MD}^2 r_{\max}^2}{\eta^2 \norm{\vec{h}_1-\vec{h}_2}^2} - \norm{\vec{h}_2}^2}- \sqrt{\frac{C_{MD}^2 r^2}{\eta^2 \norm{\vec{h}_1-\vec{h}_2}^2} - \norm{\vec{h}_2}^2}\\
    &= \frac{\frac{C_{MD}^2 (r_{\max}^2 - r^2)}{\eta^2 \norm{\vec{h}_1-\vec{h}_2}^2}}{\sqrt{\frac{C_{MD}^2 r_{\max}^2}{\eta^2 \norm{\vec{h}_1-\vec{h}_2}^2} - \norm{\vec{h}_2}^2} + \sqrt{\frac{C_{MD}^2 r^2}{\eta^2 \norm{\vec{h}_1-\vec{h}_2}^2} - \norm{\vec{h}_2}^2}}\\
    &= \frac{C_{MD}(r_{\max}-r)}{\eta \norm{\vec{h}_1-\vec{h}_2}} \cdot \frac{\frac{C_{MD}(r_{\max}+r)}{\eta \norm{\vec{h}_1-\vec{h}_2}^2}}{\sqrt{\frac{C_{MD}^2 r_{\max}^2}{\eta^2 \norm{\vec{h}_1-\vec{h}_2}^2} - \norm{\vec{h}_2}^2} + \sqrt{\frac{C_{MD}^2 r^2}{\eta^2 \norm{\vec{h}_1-\vec{h}_2}^2} - \norm{\vec{h}_2}^2}}\\
    &\leq \frac{C_{MD}(r_{\max}-r)}{\eta \norm{\vec{h}_1-\vec{h}_2}} \cdot \frac{\frac{2C_{MD}r_{\max}}{\eta \norm{\vec{h}_1-\vec{h}_2}}}{\sqrt{\frac{C_{MD}^2 r_{\max}^2}{\eta^2 \norm{\vec{h}_1-\vec{h}_2}^2} - \norm{\vec{h}_2}^2}}\\
    &=\frac{C_{MD}(r_{\max}-r)}{\eta \norm{\vec{h}_1-\vec{h}_2}} \cdot \frac{\frac{2C_{MD}r_{\max}}{\eta \norm{\vec{h}_1-\vec{h}_2}}}{\sqrt{\frac{C_{MD}^2 r_{\max}^2}{\eta^2 \norm{\vec{h}_1-\vec{h}_2}^2} - \frac{1}{2}\norm{\vec{h}_1-\vec{h}_2}^2}}\\
    &= \frac{C_{MD}(r_{\max}-r)}{\eta \norm{\vec{h}_1-\vec{h}_2}} \cdot \frac{\frac{2C_{MD}r_{\max}}{\eta \norm{\vec{h}_1-\vec{h}_2}}}{\sqrt{\frac{C_{MD}^2 r_{\max}^2}{\eta^2 \norm{\vec{h}_1-\vec{h}_2}^2} - \frac{1}{2}\frac{C_{MD}^2 r_{\max}^2}{\eta^2 \norm{\vec{h}_1-\vec{h}_2}^2}}}\\
    &= \frac{2\sqrt{2}C_{MD}}{ \norm{h_1-h_2}} \cdot \frac{r_{\max}-r}{\eta}. 
\end{align*}

On the other hand it also holds
\begin{align*}
    \norm{\vec{u}_1(r) - \vec{h}_1} &\geq \norm{\vec{u}_1^{**}(r) - \vec{h}_1} \\
    &= \norm{\vec{h}_2-\vec{u}_1^{**}(r_{\max})} - \norm{\vec{h}_2-\vec{u}_1^{**}(r)}\\
    &= \frac{C_{MD} r_{\max}}{\eta \norm{\vec{h}_1-\vec{h}_2}}- \frac{C_{MD} r}{\eta \norm{\vec{h}_1-\vec{h}_2}}\\
    &= \frac{C_{MD} }{\norm{\vec{h}_1-\vec{h}_2}} \cdot \frac{r_{\max}-r}{\eta}.
\end{align*}
Combining the above estimates yields the desired statement.
\end{proof}


\subsection{The behavior of $G_{\lambda,\alpha}$ near $r_{\max}$}\label{section:behavior_g}
We study the behavior of $G_{\lambda,\alpha}(\vec{W},\vec{U},r)$ as a function of $r$, where $\vec{W}$ is fixed as in Assumption \ref{Assumption:MD_model}, $\vec{U}=\vec{U}(r)$ is the optimal position discussed in Subsection \ref{section:est_solution_P_u1} and $r$ lies near $r_{\max}$.

\begin{lemma}\label{lemma:behavior_g}
Let $\vec{u}_1(r)$ be the minimizer of (\ref{problem:u1}) with input $r \in [r_{\max}/\sqrt{2}, r_{\max}]$. Then for any $r$ such that $\frac{r_{\max}-r}{\eta} <1$, it holds
\begin{align*}
    G_{\lambda,\alpha}\Big( \vec{W},\vec{U}(r), r\Big) -G_{\lambda,\alpha}\Big( \vec{W},\vec{U}(r_{\max}), r_{\max}\Big) \geq C_1\left(\frac{r-r_{\max}}{\eta}\right)^2
\end{align*}
for some constant $C_1>0$. 
\end{lemma}

\begin{proof}
Due to symmetry, we only need to consider the half of $G_{\lambda,\alpha}$ that involves $\vec{u}_1$, i.e. the function $g: \R^M \to \R$,
\begin{align}\label{eq:def_g_u1}
    g(\vec{u}_1) = \log \Big( 1+e^{\sprod{\vec{w}_2-\vec{w}_1,\vec{u}_1}} \Big) - \frac{\alpha}{2}\sprod{\vec{w}_2-\vec{w}_1,\vec{u}_1} + \lambda \norm{\vec{u}_1}^2.
\end{align}
We approximate $g(\vec{u}_1(r))$ using the second-order Taylor approximation around $\vec{h}_1 = \vec{u}_1(r_{\max})$,
\begin{align}\label{eq:Taylor_g}
    g(\vec{u}_1) = g(\vec{h}_1) + \sprod{\nabla_{\vec{u}_1} g(\vec{h}_1), \vec{u}_1-\vec{h}_1} + \frac{1}{2}\sprod{\vec{H}_{\vec{u}_1} g(h_1)(\vec{u}_1-\vec{h}_1), \vec{u}_1-\vec{h}_1} + O(\norm{\vec{u}_1-\vec{h}_1}^3),
\end{align}
where the derivatives of $g$ (at $\vec{h}_1$) are given by
\begin{align*}
    \nabla_{\vec{u}_1} g(\vec{h}_1) &= \frac{e^{\sprod{\vec{w}_2-\vec{w}_1,\vec{h}_1}}}{1+ e^{\sprod{\vec{w}_2-\vec{w}_1,\vec{h}_1}}}(\vec{w}_2-\vec{w}_1) - \frac{\alpha}{2}(\vec{w}_2-\vec{w}_1) + 2\lambda \vec{h}_1,\\
    \vec{H}_{\vec{u}_1}g(\vec{h}_1) &= \frac{e^{\sprod{\vec{w}_2-\vec{w}_1,\vec{h}_1}}}{(1+e^{\sprod{\vec{w}_2-\vec{w}_1,\vec{h}_1}})^2} (\vec{w}_2-\vec{w}_1) (\vec{w}_2-\vec{w}_1)^\top + 2\lambda \vec{I}. 
\end{align*}
Since $\vec{u}_1 = \vec{h}_1$ is the minimum of $g(\vec{u}_1)$ under the constraint $\vec{u}_1 \geq 0$ and $\vec{u}_1(r)$ is always feasible for any $r$, the linear term in (\ref{eq:Taylor_g}) is non-negative, i.e.
\begin{align*}
    \sprod{\nabla _{\vec{u}_1} g(\vec{h}_1), \vec{u}_1-\vec{h}_1} \geq 0.
\end{align*}

Next, we consider the second-order term in (\ref{eq:Taylor_g}). We have
\begin{align*}
    \sprod{H_{\vec{u}_1}g (\vec{h}_1)  (\vec{u}_1-\vec{h}_1), \vec{u}_1-\vec{h}_1} &> \frac{e^{\sprod{\vec{w}_2-\vec{w}_1,\vec{h}_1}}}{(1+e^{\sprod{\vec{w}_2-\vec{w}_1,\vec{h}_1}})^2} \sprod{\vec{w}_2-\vec{w}_1, \vec{u}_1-\vec{h}_1}^2\\
    &\geq \frac{e^{\sprod{\vec{w}_2-\vec{w}_1,\vec{h}_1}}}{2(1+e^{\sprod{\vec{w}_2-\vec{w}_1,\vec{h}_1}})^2} \norm{\vec{w}_2-\vec{w}_1}^2 \norm{\vec{u}_1-\vec{h}_1}^2,
\end{align*}
where the last inequality holds because the angle between $\vec{w}_2-\vec{w}_1$ and $\vec{u}_1-\vec{h}_1$ lies between $0$ and $\pi/4$, which follows directly from Lemma \ref{Lemma:P_u2}.

Inserting the above observations back into the Taylor expansion (\ref{eq:Taylor_g}) and applying Lemma \ref{lemma:problem_u3}, we obtain
\begin{align*}
    g(\vec{u}_1(r)) - g(\vec{h}_1) &>\Bigg(\frac{e^{\sprod{\vec{w}_2-\vec{w}_1,\vec{h}_1}}}{2(1+e^{\sprod{\vec{w}_2-\vec{w}_1,\vec{h}_1}})^2} \norm{\vec{w}_2-\vec{w}_1}^2 +2\lambda\Bigg)\norm{\vec{u}_1(r) - \vec{h}_1}^2 \\
    &\geq C_1\left(\frac{r-r_{\max}}{\eta}\right)^2
\end{align*}
where the constant $C_1$ is given by 
\begin{align*}
    C_1 = \Big(\frac{e^{\sprod{\vec{w}_2-\vec{w}_1,\vec{h}_1}}}{2(1+e^{\sprod{\vec{w}_2-\vec{w}_1,\vec{h}_1}})^2} \norm{\vec{w}_2-\vec{w}_1}^2 + 2\lambda\Big) c^2    
\end{align*}
with $c$ from Lemma \ref{lemma:problem_u3}, i.e. 
\begin{align*}
    C_1 = \Big(\frac{e^{\sprod{\vec{w}_2-\vec{w}_1,\vec{h}_1}}}{2(1+e^{\sprod{\vec{w}_2-\vec{w}_1,\vec{h}_1}})^2} \norm{\vec{w}_2-\vec{w}_1}^2 + 2\lambda\Big)     \frac{C_{MD}^2}{\norm{\vec{h}_1-\vec{h}_2}^2}.
\end{align*}

\end{proof}

\subsection{The behavior of $F_{\lambda,\alpha}$} \label{section:behavior_f}
In this subsection we study the behavior of the function $F_{\lambda, \alpha}$ under the assumptions in Assumption \ref{Assumption:MD_model}. 

\begin{lemma}\label{lemma:behavior_f}
For $r< 1$, the function $r \mapsto F_{\lambda,\alpha}(\vec{W},\vec{H},r)$ satisfies
\begin{align*}
    F_{\lambda,\alpha}(\vec{W},\vec{H},r) - F_{\lambda,\alpha}(\vec{W},\vec{H},0) \leq C_2 r^2 
\end{align*}
for some constant $C_2>0$.
\end{lemma}

\begin{proof}
By symmetry we only need to consider the half of $F_{\lambda,\alpha}$ that involves $\vec{h}_1$, and to simplify notations we denote this by $\tilde{F} = \tilde{F}(r)$ with
\begin{align*}
    \tilde{F}(r) &= \int \Bigg(\ell_\alpha \Big(\vec{W},\vec{h}_1 + \vec{v}, \vec{y}_1^{(\alpha)} \Big)+ \lambda \norm{\vec{h}_1+\vec{v}}^2 \Bigg) d\mu_r^1(\vec{v})\\
    &=\int \Bigg(\log \Big(1+e^{\sprod{\vec{w}_2-\vec{w}_1,\vec{h}_1+\vec{v}}}\Big) -\frac{\alpha}{2}\sprod{\vec{w}_2-\vec{w}_1,\vec{h}_1+\vec{v}} + \lambda\norm{\vec{h}_1+\vec{v}}^2\Bigg) d\mu_r^1(\vec{v})\\
    &=\int \Bigg(\log \Big(1+e^{\sprod{\vec{w}_2-\vec{w}_1,\vec{h}_1+\vec{v}}}\Big)  + \lambda\norm{\vec{v}}^2\Bigg)d\mu_r^1(\vec{v}) - \frac{\alpha}{2}\sprod{\vec{w}_2-\vec{w}_1,\vec{h}_1} +\lambda \norm{\vec{h}_1}^2,
\end{align*}
where the last equality comes from the second statement in Assumption \ref{Assumption:MD_model}.
We denote the integrand in the above formulation by $\tilde{f}$,  
\begin{align*}
    \tilde{f}(\vec{v}) = \log \Big( 1+ e^{\sprod{\vec{w}_2-\vec{w}_1,\vec{h}_1+\vec{v}}}\Big) + \lambda \norm{\vec{v}}^2,
\end{align*}

Now we approximate $\tilde{f}$ using its second-order Taylor expansion, which yields a rest of order $O(\norm{v}^3)$. From the second statement in Assumption \ref{Assumption:MD_model}, $\norm{v}$ is upper bounded by $Ar$ and hence the rest of the Taylor approximation is of order $O(r^3)$. Hence we obtain 
\begin{align*}
    \tilde{f}(\vec{v}) &= \tilde{f}(0) + \sprod{\nabla \tilde{f}(0), \vec{v}} + \frac{1}{2}  \sprod{H\tilde{f}(0) \vec{v}, \vec{v}} + O(r^3)\\
    &= \tilde{f}(0) + \frac{e^{\sprod{\vec{w}_2-\vec{w}_1,\vec{h}_1}}}{1+e^{\sprod{\vec{w}_2-\vec{w}_1,\vec{h}_1}}}\sprod{\vec{w}_2-\vec{w}_1, \vec{v}} \\&\quad + \frac{e^{\sprod{\vec{w}_2-\vec{w}_1,\vec{h}_1}}}{2\Big( 1+ e^{\sprod{\vec{w}_2-\vec{w}_1,\vec{h}_1}} \Big)^2} \sprod{ \vec{w}_2-\vec{w}_1, \vec{v}}^2 + \lambda \norm{\vec{v}}^2 + O(r^3).
\end{align*}

Taking the integral $\int d\mu_1^r(\vec{v})$ we see that again due to the second statement in Assumption \ref{Assumption:MD_model}, the first order term in the Taylor expansion of $\tilde{f}$ vanishes. Therefore we obtain
\begin{align*}
    \tilde{F}(r) -\tilde{F}(0)&= \int \tilde{f}(\vec{v}) d\mu_1^r(\vec{v}) \\
    &= \int \Bigg(\frac{e^{\sprod{\vec{w}_2-\vec{w}_1,\vec{h}_1}}}{2\Big( 1+ e^{\sprod{\vec{w}_2-\vec{w}_1,\vec{h}_1}} \Big)^2} \sprod{ \vec{w}_2-\vec{w}_1, v}^2 + \lambda \norm{\vec{v}}^2 \Bigg) d\mu_1^r(\vec{v}) + O(r^3)\\
    &\leq \Bigg(\frac{e^{\sprod{\vec{w}_2-\vec{w}_1,\vec{h}_1}} \norm{\vec{w}_2-\vec{w}_1}^2}{2\Big( 1+ e^{\sprod{\vec{w}_2-\vec{w}_1,\vec{h}_1}} \Big)^2}+ \lambda \Bigg) \int \norm{\vec{v}}^2 d\mu_1^r(\vec{v}) + O(r^3)\\
    &\leq A^2\Bigg(\frac{e^{\sprod{\vec{w}_2-\vec{w}_1,\vec{h}_1}} \norm{\vec{w}_2-\vec{w}_1}^2}{2\Big( 1+ e^{\sprod{\vec{w}_2-\vec{w}_1,\vec{h}_1}} \Big)^2}+ \lambda \Bigg)  r^2 + O(r^3),
\end{align*}
where $A$ is the constant for which $\norm{\vec{v}} \leq Ar$ holds (see Assumption \ref{Assumption:MD_model}). Thus the desired statement follows with 
\begin{align*}
    C_2 =  A^2\Bigg(\frac{e^{\sprod{\vec{w}_2-\vec{w}_1,\vec{h}_1}} \norm{\vec{w}_2-\vec{w}_1}^2}{2\Big( 1+ e^{\sprod{\vec{w}_2-\vec{w}_1,\vec{h}_1}} \Big)^2}+ \lambda \Bigg)  .
\end{align*}

\end{proof}

\subsection{Estimation of the optimal dilation $r_*$}\label{section:optimal_dilation}
In this subsection we come back to the $\mathcal{MD}$ problem, i.e. the minimization of the MD risk
\begin{align*}
    \min_{U,r} \quad F_{\lambda,\alpha}(\vec{W},\vec{H},r) + \eta G_{\lambda,\alpha} (\vec{W},\vec{U},r) \quad \text{s.t}\quad  (\ref{eq:MD_condition1}),(\ref{eq:MD_condition2}).
\end{align*}
As discussed in Subsection \ref{Section:MD_prepare_proof} we can simplify this problem by inserting into $U$ the solution $U(r)$ to the problem
\begin{align*}
    \min_{U} \quad G_{\lambda,\alpha}(\vec{W},\vec{U},r) \quad \text{s.t.} \quad  (\ref{eq:MD_condition1}),(\ref{eq:MD_condition2}).
\end{align*}
Then the $\mathcal{MD}$ problem is reduced to the minimization over the dilation $r$, namely
\begin{align}\label{Problem:MD_reduced_1}
    \min_{r} \quad F_{\lambda,\alpha}(\vec{W},\vec{H},r) + \eta G_{\lambda,\alpha} \big(\vec{W},\vec{U}(r),r\big).
\end{align}

\begin{lemma}\label{lemma:optimal_dilation}
The solution $r_*$ to the problem (\ref{Problem:MD_reduced_1}) satisfies $r_{\max} \geq r_* \geq r_{\max} (1-C' \eta ^{1/2})$, where $C'$ is a constant given by 
\begin{align*}
     C' := \frac{A\sqrt{2}\norm{\vec{h}_1-\vec{h}_2}}{C_{MD}}.
\end{align*}
\end{lemma}

\begin{proof}

To shorten notations, we define
\begin{align*}
    f(r) := F_{\lambda,\alpha}(\vec{W},\vec{H},r) \quad \text{and} \quad g(r) := G_{\lambda,\alpha} \big(\vec{W},\vec{U}(r),r\big).
\end{align*}
Then the problem (\ref{Problem:MD_reduced_1}) can be rewritten as
\begin{align}\label{Problem:MD_reduced_2}
    \min_{r} \quad f(r) + \eta g(r). 
\end{align}

First we observe that $r_* \leq r_{\max}$ because for any $r> r_{\max}$ we have that $f(r) > f(r_{\max})$ while $g(r) = g(r_{\max})$ (see Remark \ref{remark:h1_opt}). Thus we only need to show the lower bound on $r_*$. 

Let $\epsilon \in \left(0,\frac{C_{MD}}{\norm{h_1-h_2}^2}\right) $, we will show that the solution to the reduced $\mathcal{MD}$ problem (\ref{Problem:MD_reduced_2}) cannot be $r$ for any $r <(1-\epsilon) r_{\max}$, provided that $\epsilon$ is sufficiently large (this condition will be later specified more precisely). To see this we will show that for any such $r$ it holds
\begin{align}\label{eq:goal}
    f(r)+ \eta g(r) > f(r_{\max}) + \eta g(r_{\max}).
\end{align}

By Lemma \ref{lemma:behavior_f} we have that 
\begin{align*}
    f(r_{\max}) - f(r) \leq f(r_{\max}) -f(0) \leq C_2 r_{\max}^2.
\end{align*}

On the other hand, from Lemma \ref{lemma:behavior_g} it follows that
\begin{align*}
    g(r) - g(r_{\max}) \geq g\Big( (1-\epsilon)r_{\max}\Big) - g(r_{\max}) \geq C_1 \frac{\epsilon^2 r_{\max}^2}{\eta^2}
\end{align*}
holds for some constant $c_2 >0$. 

Combining the above observations we see that (\ref{eq:goal}) will hold if 
\begin{align*}
    C_1 \frac{\epsilon^2 r_{\max}^2}{\eta} \geq C_2 r_{\max}^2,
\end{align*}
which holds provided that
\begin{align*}
    \epsilon \geq C' \eta^{1/2}
\end{align*}
with 
\begin{align*}
    C' =  \frac{A\sqrt{2}\norm{\vec{h}_1-\vec{h}_2}}{C_{MD}}.
\end{align*}

Since any candidate outside the interval $\Big[ r_{\max}(1-C' \eta^{1/2}), r_{\max}\Big]$ is worse than $r_{\max}$, we conclude that $r_*$ must be in this interval. 

\end{proof}

\subsection{Finalizing the proof}\label{section:finalize_proof}
In previous subsections we have approximately estimated the optimal dilation $r_*$ of the $\mathcal{MD}$ problem in general, i.e. the LS parameter $\alpha$ can take any value in $\{0, \alpha_0\}$. Now we distinguish between the two values of $\alpha$ by adding the superscripts $LS$ (corresponding to $\alpha = \alpha_0$) and $CE$ (corresponding to $\alpha = 0$), and we will finalize the proof of Theorem \ref{theorem:MD_model} by showing 
\begin{align}\label{eq:desired_MD}
    \frac{r^{CE}_*}{\norm{\vec{h}_1^{CE} - \vec{h}_2^{CE}}} > \frac{r^{LS}_*}{\norm{\vec{h}_1^{LS} - \vec{h}_2^{LS}}}.
\end{align}

By Assumption \ref{Assumption:MD_model} we have $\norm{\vec{h}_1^{CE} - \vec{h}_2^{CE}} = \gamma \norm{\vec{h}_1^{LS} - \vec{h}_2^{LS}}$, hence
\begin{align*}
    \frac{r^{CE}_{\max}}{\norm{\vec{h}_1^{CE} - \vec{h}_2^{CE}}} = \frac{\eta \norm{\vec{h}_1^{CE} - \vec{h}_2^{CE}}}{C_{MD}}= \gamma \frac{\eta \norm{\vec{h}_1^{LS} - \vec{h}_2^{LS}}}{C_{MD}} = \gamma \frac{r^{LS}_{\max}}{\norm{\vec{h}_1^{LS} - \vec{h}_2^{LS}}}.
\end{align*}
Combining this and Lemma \ref{lemma:optimal_dilation} we obtain
\begin{align*}
    \quad \frac{r^{CE}_{*}}{\norm{\vec{h}_1^{CE} - \vec{h}_2^{CE}}} 
    &> \frac{\Big(1-C'\eta^{1/2}\Big)r^{CE}_{\max}}{\norm{\vec{h}_1^{CE} - \vec{h}_2^{CE}}} \\
    &= \gamma \Big(1-C'\eta^{1/2}\Big)\frac{r^{LS}_{\max}}{\norm{\vec{h}_1^{LS} - \vec{h}_2^{LS}}}\\
    &\geq \gamma \Big(1-C'\eta^{1/2}\Big)\frac{r^{LS}_{*}}{\norm{\vec{h}_1^{LS} - \vec{h}_2^{LS}}}.
\end{align*}
Hence (\ref{eq:desired_MD}) holds provided that $\gamma \Big(1-C'\eta^{1/2}\Big) \geq 1$, which follows from the third statement in Assumption \ref{Assumption:MD_model}. 
\end{document}